\documentclass{article}
\usepackage[margin=1in]{geometry}
\usepackage[utf8]{inputenc}
\usepackage{amsthm,amsfonts,amsmath,amssymb,xcolor,hyperref,enumerate,parskip}
\usepackage{aligned-overset}
\usepackage{chngcntr}
\usepackage{algorithm}
\usepackage{algorithmic}

\allowdisplaybreaks
\usepackage{caption,subcaption}
\UseRawInputEncoding
\usepackage[square,numbers]{natbib}
\newtheorem{theorem}{Theorem}[section]
\newtheorem{definition}{Definition}
\newtheorem{proposition}[theorem]{Proposition}

\newtheorem{assumption}{Assumption}

\newtheorem{lemma}[theorem]{Lemma}

\newcommand{\CLSI}[0]{C_{\mathsf{LSI}}}

\newcommand{\sep}[0]{\,||\,}

\title{Mirror Mean-Field Langevin Dynamics}
\author{
Anming Gu\thanks{Equal contribution.}\\
University of Texas at Austin \\ 
{\texttt{anminggu@cs.utexas.edu}}
\and
Juno Kim\footnotemark[1]\\
UC Berkeley \\
{\texttt{junokim@berkeley.edu}}
}
\date{\today}

\makeatletter
\def\thm@space@setup{%
  \thm@preskip=\parskip \thm@postskip=0pt
}
\makeatother

\newcommand{\KL}[0]{\mathsf{KL}}
\newcommand{\FI}[0]{\mathsf{FI}}
\newcommand{\Ent}[0]{\mathsf{Ent}}
\newcommand{\R}[0]{\mathbb{R}}
\newcommand{\E}[0]{\mathbb{E}}
\newcommand{\Tr}{\mathsf{Tr}}

\begin{document}
\maketitle

\begin{abstract}
    The mean-field Langevin dynamics (MFLD) minimizes an entropy-regularized nonlinear convex functional on the Wasserstein space over $\mathbb{R}^d$, and has gained attention recently as a model for the gradient descent dynamics of interacting particle systems such as infinite-width two-layer neural networks. However, many problems of interest have constrained domains, which are not solved by existing mean-field algorithms due to the global diffusion term. We study the optimization of probability measures constrained to a convex subset of $\mathbb{R}^d$ by proposing the \emph{mirror mean-field Langevin dynamics} (MMFLD), an extension of MFLD to the mirror Langevin framework. We obtain linear convergence guarantees for the continuous MMFLD via a uniform log-Sobolev inequality, and uniform-in-time propagation of chaos results for its time- and particle-discretized counterpart.
\end{abstract}

\section{Introduction}
In this work, we study the problem of minimizing an entropy-regularized functional on the space of probability measures constrained to a convex subset $\mathcal{X}\subseteq\mathbb{R}^d$:
\begin{equation}\label{eq:mfld_L_intro}
\mathcal{L}(\mu) := F(\mu) + \lambda {\mathsf{Ent}}(\mu),
\end{equation}
where $\mu\in\mathcal{P}_2(\mathcal{X})$, $F: \mathcal{P}_2(\mathcal{X})\to\mathbb{R}$ is a linearly convex functional, and $\mathsf{Ent}(\mu) = \int \log\frac{d\mu}{dx}d\mu$ is the entropy of $\mu$. Such problems arise naturally in statistical inference and machine learning, such as when studying infinite-width limits of neural networks \citep{mei2018meanfield,nitanda2022convex,suzuki2023feature,takakura2024meanfield}, tensor decomposition \citep{chizat2018global}, sparse spikes deconvolution \citep{chizat2018global}, density estimation \citep{suzuki2023meanfield} and discrepancy minimization \citep{gretton12,suzuki2023meanfield}. Often, \emph{nonconvex} optimization of a particle system can be formulated as an instance of \eqref{eq:mfld_L_intro} by lifting to the space of measures.

In the simplest case, we wish to sample from a distribution $\mu_*\propto e^{-f/\lambda}$ for a potential $f:\mathcal{X}\to\mathbb{R}$, which corresponds to \eqref{eq:mfld_L_intro} where $F(\mu)=\int fd\mu$ is linear. When $\mathcal{X}=\mathbb{R}^d$, a classical approach is to employ a discretization of the Langevin dynamics
\begin{equation}
dX_t = -\nabla f(X_t)dt + \sqrt{2\lambda}dB_t,\label{eq:langevin_dynamics}    
\end{equation}
where $B_t$ is standard Brownian motion. The trajectory $\mu_t = \operatorname{Law}(X_t)$ can be interpreted as the Wasserstein gradient flow minimizing the KL divergence with respect to $\mu_*$ \citep{Jordan98}; this remarkable connection has led to many fruitful developments in both sampling and optimization, see for instance \citet{Wibisono2018sampling,chewi2023logconcavesampling} for an exposition. The convergence of \eqref{eq:langevin_dynamics} in various metrics (Wasserstein, TV, KL, $\chi^2$, R\'enyi) has been extensively analyzed under strong log-concavity or functional inequalities such as log-Sobolev or Poincar\'e
inequalities \citep{Wibisono2018sampling,wibisono2019proximal,vempalawibisono,durmus17,durmus19,li2021,chewi22,mousavi2023}.

In the constrained or ill-conditioned case, however, \eqref{eq:langevin_dynamics} cannot be applied as the diffusion term will cause mass to escape $\mathcal{X}$. This problem can be solved by imposing an alternative non-flat geometry on $\mathcal{X}$ and running the mirror Langevin dynamics (MLD), the analogue of mirror descent for sampling \citep{hsieh2018,zhang2020wasserstein}. By introducing a barrier function $\phi:\mathcal{X}\to\mathbb{R}$, the mirror map $\nabla\phi$ induces an isometry between the primal space $\mathcal{X}$ with the Hessian metric $\nabla^2\phi$, and the dual space $\mathbb{R}^d$ with the metric $\nabla^2\phi^*$. This allows us to apply a transformed diffusion in the dual space:
\begin{equation}\label{eq:mld_intro}
\begin{split}
X_t &= \nabla\phi^*(Y_t),\\
dY_t &= -\nabla f(X_t)\,dt \;+\; \sqrt{2\lambda \nabla^2\phi(X_t)}\, dB_t.
\end{split}
\end{equation}
The convergence of \eqref{eq:mld_intro} under a mirror Poincar\'{e} inequality have been established in \citet{chewi2020mld} and various discretization schemes have been analyzed in \citet{zhang2020wasserstein,jiang2021mirror,ahn2021efficient,li2022mirror}. In particular, \citet{ahn2021efficient} proposed a discretized mirror Langevin algorithm which ensures vanishing bias under mild assumptions on the mirror map.

Returning to the general problem \eqref{eq:mfld_L_intro}, when $\mathcal{X}=\mathbb{R}^d$, this can be solved efficiently by running the mean-field Langevin dynamics (MFLD) \citep{nitanda2022convex,chizat2022meanfield} extending \eqref{eq:langevin_dynamics}, given by the McKean-Vlasov process
\begin{equation}\label{eq:mfld_sde_intro}
    dX_t = -\nabla \frac{\delta F(\mu_t)}{\delta\mu}(X_t)dt + \sqrt{2\lambda}dB_t,
\end{equation}
which is the Wasserstein gradient flow minimizing $\mathcal{L}$. The convergence of \eqref{eq:mfld_sde_intro} and its time discretization has been established under a log-Sobolev inequality (LSI) in \citep{nitanda2022convex}. Following works have studied \emph{propagation of chaos} \citep{sznitman1991}, that is, the approximation error induced when $\mu_t$ is replaced by a finite particle system. A dedicated line of work has obtained error bounds which converge to zero uniformly for all time as the number of particles $N\to\infty$ by leveraging uniform LSI \citep{chen2023uniform,suzuki2023uniformintime,suzuki2023meanfield}. Recent works have further shown that these errors can be made independent of the LSI constant \citep{nitanda2024improved,chewi2024uniform,nitanda2025}.

The same analyses may be applied if $\mathcal{X}$ is a complete Riemannian manifold without boundary, e.g., a hypersphere or torus. However, the problem for general $\mathcal{X}$ has remained open. Nonetheless, many applications of MFLD require the domain of optimization to be (implicitly or explicitly) constrained, e.g., trajectory inference \citep{chizat2022trajectoryinferencemeanfieldlangevin,gu2025partially,gu2025private}, computation of Wasserstein barycenters \citep{chizat2023doublyregularizedentropicwasserstein,vaskevicius2023barycenter}, computation of discrepancy measures \citep{suzuki2023meanfield} or signal deconvolution \citep{chizat2018global} with bounded support, and optimization of neural networks \citep{nitanda2022convex,nitanda2025} with constrained parameters. A more detailed overview of related works is provided in Appendix \ref{app:related}.

\subsection{Contributions}


We propose the \emph{mirror mean-field Langevin dynamics} (MMFLD), which unifies the mirror \eqref{eq:mld_intro} and mean-field \eqref{eq:mfld_sde_intro} analyses to solve the constrained distributional optimization problem \eqref{eq:mfld_L_intro} for general convex $\mathcal{X}\subseteq \mathbb{R}^d$ and (linear) convex $F:\mathcal{P}_2(\mathcal{X})\to\mathbb{R}^d$. To the best of our knowledge, this constitutes the first algorithm minimizing \eqref{eq:mfld_L_intro} with global convergence guarantees. In particular, we provide non-asymptotic convergence rates for both the continuous-time mean-field flow (Section \ref{sec:mirror_mfld}), and the time- and particle-discretized algorithm (Section \ref{sec:discretized_analysis}). Crucially, we show the recent advances in propagation of chaos analysis extend gracefully to constrained domains.

\paragraph{Notation.} The Euclidean norm on $\mathbb{R}^d$ is denoted as $\|\cdot\|$. We use $\int f$ to denote the integral with respect to Lebesgue measure: $\int f\,dx$. When the integral is with respect to a different measure $\mu$, we write it as $\int f\, d\mu$. We use $X, Y$ for random variables and non-tilde $\mu$ and tilde $\Tilde\mu$ for probability distributions in the primal and dual spaces, respectively. When it is clear from context, we will sometimes abuse notation by identifying a measure with its density. We also use the same symbol $\CLSI$ to denote log-Sobolev constants for LSI, mirror LSI, and uniform-in-$N$ mirror LSI, etc. $\delta_x$ denotes the Dirac delta measure at $x\in\mathbb{R}^d$. For a measurable function $T:\mathcal{X}\to\mathcal{Y}$, we use $T_\sharp \mu$ to denote the pushforward of $\mu$ by $T$, e.g. $T_\sharp\mu(B) = \mu(T^{-1}(B))$ for any measurable set $B\subseteq \mathcal{Y}$. Finally, we use $\mathcal{P}_2(\mathcal{X})$ to denote the set of Borel probability measures over $\mathcal{X}$ with finite second moments, equipped with the 2-Wasserstein metric.

\section{Preliminaries}\label{sec:prelim}

In order to develop our main object of interest, the mirror mean-field Langevin dynamics, we first provide discussion of the mirror Langevin dynamics and mean-field Langevin dynamics as a primer.

\subsection{Mirror Langevin Dynamics}\label{sec:mirror_ld}
The mirror Langevin dynamics (MLD) has been extensively studied in the context of constrained sampling problems \citep{zhang2020wasserstein,jiang2021mirror,li2022mirror}. Suppose we want to sample from a probability distribution~$\mu$ supported on a convex set $\mathcal{X}\subseteq \mathbb{R}^d$. We assume $\mu$ is absolutely continuous w.r.t. Lebesgue measure on $\mathbb{R}^d$ and has density $\mu\propto e^{-f/\lambda}$ for some differentiable $f:\mathcal{X}\to\mathbb{R}$. 

Certainly, one way is to run the Langevin dynamics (or algorithm) \eqref{eq:langevin_dynamics} and project onto $\mathcal{X}$. However, one undesirable behavior of this projection step is that this will necessarily put positive mass onto $\partial\mathcal{X}$; see the experiments in Section \ref{sec:exp}. Alternatively, we can enforce the distribution to stay in $\mathcal{X}$ by changing the geometry of the underlying space. 

Towards that end, let $\phi:\mathcal{X}\to\mathbb{R}$ be a thrice-differentiable strictly convex function of Legendre type \citep{rockafellar1997convex}. We require $\|\nabla \phi(x)\|\to \infty$ and $\nabla^2\phi(x)\to\infty$ as $x$ approaches $\partial\mathcal{X}$, as this ensures the diffusion remains inside the domain $\mathcal{X}$. We call $\nabla\phi:\mathcal{X}\to\mathbb{R}^d$ the \emph{mirror map} and $\mathcal{Y}=\nabla\phi(\mathcal{X})=\mathbb{R}^d$ the \emph{dual space}. Further, define $\phi^*:\mathbb{R}^d\to\mathbb{R}$ to be the Legendre dual (convex conjugate) of $\phi$, given as $\phi^*(y) = \sup_{x\in\mathcal{X}}\langle x,y\rangle - \phi(x)$. In Appendix \ref{app:proofs}, we provide discussion on properties of $\phi$.

The mirror Langevin dynamics (MLD) in primal space satisfies the SDE \begin{equation}\label{eq:mld_primal}
\begin{cases}
    X_t &= \nabla\phi^*(Y_t)\\
    dY_t &= -\nabla f(X_t)dt + \sqrt{2\lambda\nabla^2\phi(X_t)}dB_t,
\end{cases}
\end{equation}
and its density $\mu_t = \operatorname{Law}(X_t)$ satisfies the Fokker-Planck PDE \[
\frac{\partial \mu_t}{\partial t} = \lambda \nabla \cdot \left(\mu_t [\nabla^2\phi]^{-1}\nabla \log \frac{\mu_t}{\mu}\right).
\]
This is equivalent to the Riemmanian Langevin dynamics in the primal space with metric given by the Hessian $\nabla^2\phi$, or in the dual space with metric $\nabla^2\phi^*$, as $(\mathcal{X},\nabla^2\phi)$ is isometric to $(\mathbb{R}^d,\nabla^2\phi^*)$. This ensures we obtain the same the convergence guarantees in both the primal and dual spaces.

We will require standard Lipschitz conditions on $f$ \citep{øksendal2010stochastic} and $\|[\nabla^2\phi(x)]^{1/2}-[\nabla^2\phi(x')]^{1/2}\|_\mathsf{F}\le O(\|x-x'\|_2)$ for every $x,x'\in \mathcal{X}$. This ensures well-posedness of \eqref{eq:mld_primal}; see Appendix A of 
\citet{zhang2020wasserstein} for more details.\footnote{In the sequel, we will also implicitly assume this Frobenius norm bound for all of our results on the mirror mean-field Langevin dynamics as well.}

Using It\^o's formula, we can obtain a representation of \eqref{eq:mld_primal} purely in terms of $X_t$; however, it depends on the third-derivative of $\phi$, see \citep[Exercise 10.1]{chewi2023logconcavesampling}. Alternatively, we can reformulate the MLD entirely in the dual space as the following SDE:
\begin{equation*}
dY_t = - \nabla f(\nabla\phi^*(Y_t))dt + \sqrt{2\lambda[\nabla^2\phi^*(Y_t)]^{-1}}dB_t.
\end{equation*}

The distribution $\Tilde{\mu}_t = \operatorname{Law}(Y_t)$ is the pushforward of $\mu_t$ under the mirror map: $\Tilde{\mu}_t = (\nabla\phi)_\sharp \mu_t$.

To initiate the isoperimetric analysis, we first introduce the entropy, Kullback-Leibler (KL) divergence, and relative Fisher information. Denote the entropy of a nonnegative functional $f\ge 0$ as $\mathsf{Ent}_\mu(f) :=\mathbb{E}_\mu[f\log f] - \mathbb{E}_{\mu}[f^2]\log\mathbb{E}_\mu[f^2]$, 
KL divergence between two measures $\mu,\nu$ as $\KL(\mu\sep\nu) := \mathbb{E}_\nu\left[\frac{\mu}{\nu}\log\frac{\mu}{\nu}\right] = \int \mu\log \frac{\mu}{\nu}$, and the relative Fisher information (modified by the mirror map) as $
\FI_\phi(\mu\sep\nu) := \mathbb{E}_\mu\left\langle\nabla \log \frac{\mu}{\nu},[\nabla^2(\phi)]^{-1}\nabla \log \frac{\mu}{\nu}\right\rangle$.\footnote{In the Euclidean setting, we drop the dependence on the mirror map and use $\FI(\mu\sep\nu)$.}

We require the stationary distribution $\mu_*$ to satisfy a mirror log-Sobolev inequality:
\begin{assumption}[Mirror log-Sobolev inequality]\label{asmp:lsi}
    There exists a constant $\CLSI> 0$ such that $\mu_*$ satisfies a mirror log-Sobolev inequality with constant $\CLSI$, that is, for any smooth function $g:\mathbb{R}^d\to\mathbb{R}$, we have
    \begin{align*}
    \mathsf{Ent}_{\mu_*}(g^2) \le \frac{2}{\CLSI}\mathbb{E}_{\mu_*}\left[\|\nabla g(x)\|^2_{[\nabla^2\phi(x)]^{-1}}\right].
    \end{align*}
    Equivalently, by setting 
    $g := \sqrt{\frac{d\mu}{d\mu_*}}$, for every $\mu \in \mathcal{P}_2(\mathcal{X})$, it holds that $\KL(\mu\sep\mu_*) \le \frac{1}{2\CLSI}\FI_\phi(\mu\sep\mu_*)$.
\end{assumption}

A simple method to verify the mirror LSI is as follows: the mirror LSI is satisfied with constant $C_0/\alpha$ if the classical LSI is satisfied with constant $C_0$ and $\phi$ is $\alpha$-strongly convex \citep{Daaloul25}. We provide the proof, an explicit family of distributions satisfying the mirror LSI, and some properties of the mirror LSI in Appendix \ref{app:mlsi}. Under this assumption, we have the following convergence guarantee.
\begin{theorem}\label{thm:mld_convergence}
Let $\{\mu_t\}_{t\ge 0}$ denote the evolution of \eqref{eq:mld_primal} and $\mu_*$ denote its stationary distribution. Under Assumption \ref{asmp:lsi}, for $t\ge 0$, it holds that \[
\KL(\mu_t\sep\mu_*)\le e^{-2\CLSI\lambda t} {\KL(\mu_0\sep\mu_*)}.\]
\end{theorem}
A proof is given in \citet{jiang2021mirror}, but for completeness, we provide a short proof in Appendix \ref{sec:b1}. 

\subsection{Mean-Field Langevin Dynamics}\label{sec:mfld}
Suppose we want to minimize the following entropy-regularized functional: \begin{equation}\label{eq:mfld_L}
\mathcal{L}(\mu) := F(\mu) + \lambda{\mathsf{Ent}}(\mu),
\end{equation}
where $F: \mathcal{P}_2(\mathbb{R}^d)\to\mathbb{R}$ is a linearly convex functional (Definition \ref{def:linear_convex}), $\mathsf{Ent}(\mu):= \int \log\frac{d\mu}{dx}d\mu$ is the entropy of $\mu$, and $\lambda>0$ is the temperature parameter. We start by defining the first variation of $F$.

\begin{definition}[First variation]
We say $F:\mathcal{P}_2(\mathbb{R}^d)\to\mathbb{R}$ admits a first variation at $\mu\in \mathcal{P}_2(\mathbb{R}^d)$ if there exists a continuous function $\frac{\delta F(\mu)}{\delta\mu}:\mathbb{R}^d\to\mathbb{R}$ such that for any $\mu,\mu'\in\mathcal{P}_2(\mathbb{R}^d)$, we have 
    \begin{equation}
        \frac{dF(\mu + \epsilon(\mu' - \mu))}{d\epsilon}\bigg|_{\epsilon=0} = \int \frac{\delta F(\mu)}{\delta\mu}(\mu'-\mu)dx. \label{eq:first_variation}
    \end{equation}
    If the first variation of $F$ exists, it is unique up to a constant.
\end{definition}
To solve \eqref{eq:mfld_L}, we consider the mean-field Langevin dynamics (MFLD): \begin{equation}\label{eq:mfld_sde}
    dX_t = -\nabla \frac{\delta F(\mu_t)}{\delta\mu}(X_t)dt + \sqrt{2\lambda}dB_t,
\end{equation}
a McKean-Vlasov process where $\mu_t = \operatorname{Law}(X_t)$ and $B_t$ is the standard Brownian motion in $\mathbb{R}^d$. It follows that $\mu_t$ solves the nonlinear Fokker-Planck PDE 
\begin{equation}\label{eq:mfld_pde}
    \frac{\partial\mu_t}{\partial t} = \nabla \cdot \left(\mu_t \nabla \frac{\delta F(\mu_t)}{\delta \mu}\right) + \lambda\Delta \mu_t = \lambda \nabla \cdot\left(\mu_t \nabla\log \frac{\mu_t}{\hat{\mu}_t}\right),
\end{equation}
where $\hat{\mu}_t$ is the proximal Gibbs distribution associated to $\mu_t$ (see Definition \ref{def:prox_gibbs} below).
To ensure existence and convergence, the following assumptions are required.
\begin{assumption}[Lipschitz and smoothness]\label{asmp:smoothness}
 There exist constants $M_1,M_2>0$ such that for any $\mu,\mu'\in\mathcal{P}_2(\mathbb{R}^d)$, $x, x'\in\mathbb{R}^d$, it holds that $\left\|\nabla \frac{\delta F(\mu)}{\delta\mu}(x)\right\|_2\le M_1$ and
 \begin{align*}
     &\left\|\nabla \frac{\delta F(\mu)}{\delta \mu}(x) - \nabla \frac{\delta F(\mu')}{\delta\mu}(x')\right\|_2\le M_2(W_2(\mu,\mu') + \|x-x'\|_2).
 \end{align*}
 \end{assumption}
We remark that the first part of the assumption is only required for the discretization analysis. Here $W_2(\mu,\mu')$ is the $2$-Wasserstein distance between measures $\mu,\mu'\in\mathcal{P}_2(\mathcal{X})$, \[
W_2^2(\mu,\mu') := \inf_{\pi\in\Pi(\mu,\mu')}\int_{\mathcal{X}\times\mathcal{X}}\|x-x'\|^2d\pi(x,x'),\] where $\Pi(\mu,\mu') := \{\pi\in\mathcal{P}(\mathcal{X}\times\mathcal{X})\mid (P_x)_\sharp \pi = \mu,(P_y)_\sharp\pi = \mu')\}$ is the set of transport plans, and $P_x(x,y):= x$, $P_y(x,y) := y$ are the projections onto the first and second coordinates, respectively. 

\begin{definition}\label{def:linear_convex}
We say a functional $F:\mathcal{P}_2(\mathbb{R}^d)\to\mathbb{R}$ is linearly convex if for every $\mu,\nu\in\mathcal{P}_2(\mathbb{R}^d)$ and $\alpha\in[0,1]$, it holds that $F(\alpha\mu + (1-\alpha)\nu) \le \alpha F(\mu) + (1-\alpha)F(\nu)$. 
\end{definition}

\begin{assumption}\label{asmp:minimizer}
$F$ is linearly convex and \eqref{eq:mfld_L} admits a minimizer $\mu_*$.
\end{assumption}
For example, this holds if $F$ is of the form $\sum_i \ell_i\left(\int f_i d\mu\right)$ for convex losses $\ell_i$. 
Then the following can be shown:
\begin{theorem}
    Under Assumptions \ref{asmp:smoothness} and \ref{asmp:minimizer}, the minimizer $\mu_*$ of \eqref{eq:mfld_L} is unique and its density satisfies $\mu_* \propto \exp\left(-\frac{1}{\lambda}\frac{\delta F(\mu_*)}{\delta\mu}\right)$.
\end{theorem}

The convergence analysis of the MFLD relies on the proximal Gibbs distribution \citep{nitanda2022convex,chizat2022meanfield}.
\begin{definition}[Proximal Gibbs distribution]\label{def:prox_gibbs}
For each $\mu\in\mathcal{P}_2(\mathbb{R}^d)$, we define $\hat{\mu}$ to be the Gibbs distribution such that $\hat{\mu} \propto \exp\left(-\frac{1}{\lambda}\frac{\delta F(\mu)}{\delta\mu}\right)$.
\end{definition}
Note that for the \emph{linear} Langevin dynamics, the proximal Gibbs distribution coincides with the stationary distribution $\mu_*$. Next, we introduce the \emph{uniform} log-Sobolev inequality. 

\begin{assumption}[Uniform LSI]\label{asmp:ULSI}
    Suppose there exists a constant $\CLSI > 0$ such that  for every $\mu\in\mathcal{P}_2(\mathbb{R}^d)$, the proximal Gibbs distribution $\hat{\mu}$ satisfies a log-Sobolev inequality with constant $\CLSI$, that is, \[\KL(\mu\sep\hat{\mu})\le \frac{1}{2\CLSI}\FI(\mu\sep\hat{\mu}).\]
\end{assumption}
In the mean-field setting, this assumption is generally verified via the Holley-Stroock perturbation technique \citep{holleystroock} from the regularization term $\lambda$; see \citet{suzuki2023meanfield} for more discussion. \citet{nitanda2022convex,chizat2022meanfield} use this key ingredient, along with the entropy sandwich inequality (Lemma \ref{lem:entropy_sandwich}), to prove the convergence of the MFLD.
\begin{theorem}
    Let $(\mu_t)_{t\ge0}$ be the evolution described by \eqref{eq:mfld_pde}. Under Assumptions \ref{asmp:smoothness}-\ref{asmp:ULSI}, it holds that
    \begin{align*}
    \mathcal{L}(\mu_t) - \mathcal{L}(\mu_*) \le e^{-2\CLSI \lambda t}(\mathcal{L}(\mu_0) - \mathcal{L}(\mu_*)).
    \end{align*}
\end{theorem}
The discretization of MFLD has also been studied in \citet{nitanda2022convex,suzuki2023meanfield,nitanda2024improved}; we refer to these works for a detailed analysis.

\section{Mirror Mean-Field Langevin Dynamics}\label{sec:mirror_mfld}
Now with the necessary technical background, we introduce our main object of study. Suppose we want to minimize \eqref{eq:mfld_L} subject to the additional constraint that $\mu\in\mathcal{P}_2(\mathcal{X})$ for a convex set $\mathcal{X}\subseteq \mathbb{R}^d$:
\begin{equation}\label{eq:mmfld_L}
    \underset{\mu\in\mathcal{P}_2(\mathcal{X})}{\arg\min}\,\mathcal{L}(\mu) := F(\mu) + \lambda {\mathsf{Ent}(\mu)},
\end{equation}
where the domain of $F$ and its first-variation are accordingly modified from \eqref{eq:mfld_L}, e.g., now we have $F:\mathcal{P}_2(\mathcal{X})\to\mathbb{R}$, and its first variation at $\mu\in\mathcal{P}_2(\mathcal{X})$ is a function $\frac{\delta F(\mu)}{\delta\mu}:\mathcal{X}\to\mathbb{R}$. As we remarked in the introduction, restriction to a convex domain is a natural and frequently studied constraint in various optimization problems \citep{chizat2022trajectoryinferencemeanfieldlangevin,chizat2023doublyregularizedentropicwasserstein,vaskevicius2023barycenter,gu2025partially}.

For this problem, we introduce the \emph{mirror mean-field Langevin dynamics} (MMFLD), which is defined as the following (primal) SDE over $\mathcal{X}$:
\begin{equation}\label{eq:mmfld_sde}
\begin{cases}
    X_t &= \nabla\phi^*(Y_t)\\
    dY_t &= -\nabla\frac{\delta F(\mu_t)}{\delta\mu}(X_t)dt + \sqrt{2\lambda\nabla^2\phi(X_t)}dB_t,
\end{cases}
\end{equation}
where $\mu_t = \operatorname{Law}(X_t)$. Similar to the mirror Langevin and mean-field Langevin settings, it is easy to check that \eqref{eq:mmfld_pde} corresponds to the following Fokker-Planck PDE: 
\begin{equation}
\frac{\partial}{\partial t}\mu_t = \lambda \nabla \cdot \left(\mu_t [\nabla^2\phi]^{-1}\nabla \log \frac{\mu_t}{\hat{\mu}_t}\right).\label{eq:mmfld_pde}   
\end{equation}

Next, we introduce the following notions of local and dual norms, which allows us to define the notion of relative Lipschitz and smoothness.
\begin{definition}[Local and dual norms]
    Given a $C^2$ strictly convex function $\phi$, the local norm at $x\in \operatorname{int}\operatorname{dom}\phi$ with respect to $\phi$ is defined as $\|u\|_{\nabla^2\phi(x)}=\langle u, \nabla^2\phi(x)u\rangle^{1/2}$. The local dual norm is similarly defined as $\|u\|_{[\nabla^2\phi(x)]^{-1}} = \langle u, [\nabla^2\phi(x)]^{-1}u\rangle^{1/2}$.
\end{definition}

\begin{assumption}[Relative Lipschitz and smoothness]\label{asmp:rel_lip_and_smooth}
There exists constants $M_1, M_2>0$ such that for any $\mu,\mu'\in\mathcal{P}_2(\mathcal{X})$, $x,x'\in\mathcal{X}$, $\frac{\delta F(\mu)}{\delta\mu}$ is differentiable with $\left\|\nabla \frac{\delta F(\mu)}{\delta\mu}(x)\right\|_{[\nabla^2\phi(x)]^{-1}}\le M_1$ and 
    \begin{align*}
    &\left\|\nabla \frac{\delta F(\mu)}{\delta\mu}(x) - \nabla\frac{\delta F(\mu')}{\delta\mu}(x')\right\|_{[\nabla^2\phi(x')]^{-1}}\le M_2(\widetilde{W}_{2,\phi}(\mu,\mu') +\|\nabla\phi(x)-\nabla\phi(x')\|_{[\nabla^2\phi(x')]^{-1}}).
    \end{align*}
\end{assumption}
Here, the squared divergence $\widetilde{W}_{2,\phi}^2(\mu,\mu')$ is defined as
\begin{align*}
\inf_{\pi\in \Pi(\mu,\mu')} \int \|\nabla \phi(x)-\nabla\phi(x')\|_{[\nabla^2\phi(x')]^{-1}}^2d\pi.
\end{align*}

This assumption parallels Assumption \ref{asmp:smoothness} in the standard Euclidean setting and ensures the existence of the minimizer.\footnote{Note that $\widetilde{W}_{2,\phi}$ is not a metric, as it is not symmetric. However, its integrand is the second-order Taylor expansion of the dual Bregman divergence around $\nabla\phi(x')$.} Note also that these correspond to the mean-field generalizations of Assumptions 3, 6 from \citet{jiang2021mirror}. For instance, \citet[Section E.2]{ahn2021efficient} show that the quadratic potential on the simplex and $\phi(x) = -\sum_{i}\log x_i$ satisfies this condition in the linear setting, which can be straightforwardly generalized to the mean-field setting.

It is then straightforward to see that \eqref{eq:mmfld_L} admits a unique minimizer; see Appendix \ref{app:proofunique} for a proof.
\begin{theorem}\label{thm:unique}
Under Assumptions \ref{asmp:minimizer} and \ref{asmp:rel_lip_and_smooth}, \eqref{eq:mmfld_L} is well-posed and admits a unique minimizer $\mu_*$ which satisfies $\mu_* \propto \exp\left(-\frac{1}{\lambda}\frac{\delta F(\mu_*)}{\delta\mu}\right)$.
\end{theorem}
Recall that convergence analysis of the mean-field Langevin dynamics utilizes the entropy sandwich inequality \citep{nitanda2022convex,chizat2022meanfield}. We will also require this result. Remarkably, the statement will hold exactly in the constrained setting as well; see Lemma \ref{lem:entropy_sandwich}.

Now we can prove the convergence of the mirror mean-field Langevin dynamics in continuous time. The proof, which mirrors that of ordinary MFLD \citep{nitanda2022convex}, is given in Appendix \ref{app:proofconti}.
\begin{theorem}\label{thm:mmfld_conti}
    Let $\{\mu_t\}_{t\ge0}$ be the evolution described by \eqref{eq:mmfld_pde}. Under Assumptions \ref{asmp:minimizer}-- \ref{asmp:rel_lip_and_smooth},\footnote{We remark that Assumption \ref{asmp:ULSI} should here be the uniform \emph{mirror} LSI, but omit the full definition for space.} for all $t\ge 0$, it holds that \[
    \mathcal{L}(\mu_t) - \mathcal{L}(\mu_*) \le e^{-2\CLSI \lambda t}(\mathcal{L}(\mu_0) - \mathcal{L}(\mu_*)).
    \]
\end{theorem}

\section{Discretization Analysis for MMFLD}\label{sec:discretized_analysis}

\subsection{Obtaining the Discretized MMFLD}\label{sec:discrete}

We now derive the time- and space- (i.e. particle-) discretized version of MMFLD by motivating the empirical dynamics from an alternative objective. Indeed, the KL term in \eqref{eq:mfld_L} no longer makes sense in the finite-particle setting, as the negative entropy is not well-defined for discrete measures. Instead, we study the empirical system by lifting to the configuration space. Let $\mu^{(N)}\in\mathcal{P}^{(N)}$ be a distribution of $N$ particles $\mathbf{X} = (X^i)_{i=1}^N\in \mathcal{X}^N$, where $\mathcal{P}^{(N)}$ is the space of probability measures on $(\mathcal{X}^N, \mathcal{B}(\mathcal{X})^{\otimes N})$. We introduce the following objective on $\mathcal{P}^{(N)}$: \begin{equation}\label{eq:L_N}
    \mathcal{L}^{(N)}(\mu^{(N)}) := N \underset{\mathbf{X}\sim \mu^{(N)}}{\mathbb{E}}[F(\mu_\mathbf{X})]+\lambda{\mathsf{Ent}(\mu^{(N)}}).
\end{equation}
We can easily verify that if $\mu^{(N)}$ is the $N$-fold product measure of $\mu\in \mathcal{P}_2(\mathcal{X})$, then $\mathcal{L}^{(N)}(\mu^{(N)})\ge N\mathcal{L}(\mu)$ by the convexity of $F$ \citep{suzuki2023meanfield,nitanda2024improved}, and moreover the optimum of $\mathcal{L}^{(N)}$ is given by the following Gibbs distribution  $\mu_*^{(N)}$:
\begin{equation*}
    \frac{d\mu_*^{(N)}}{d\mathbf{x}}(\mathbf{x})\propto \exp\left(-\frac{N}{\lambda}F(\mathbf{x})\right).
\end{equation*}

To solve \eqref{eq:L_N}, consider the  finite-particle approximation of \eqref{eq:mmfld_sde}, which is described by the system of SDEs $\{\mathbf{X}_t\}_{t\ge0} = \{(X_t^1,\dots, X_t^N)\}_{t\ge0}$: \begin{equation}
dY_{t}^i =  -\nabla\frac{\delta F(\mu_t)}{\delta\mu}(X_t^i)dt + \sqrt{2\lambda\nabla^2\phi(X_t^i)}dB_t^i\label{eq:system}    
\end{equation}
with 
$X_{t}^i = \nabla\phi^*(Y_{t}^i)$. We sometimes denote $F(\mathbf{x}) = F(\mu_\mathbf{X})$ when emphasizing $F$ as a function of $\mathbf{x}$. Noticing that $N\nabla_{x^i}F(\mathbf{X}_t) = \nabla \frac{\delta F(\mu_\mathbf{X})}{\delta\mu}(x^i)$ \citep{chizat2022meanfield}, we can identify the system \eqref{eq:system} with
\begin{align*}
    \begin{cases}
    \mathbf{X}_t &= \nabla\phi^*(\mathbf{Y}_t),\\
        d\mathbf{Y}_t &= -N\nabla F(\mathbf{X}_t)dt + \sqrt{2\lambda[\nabla^2\phi^*(\mathbf{X}_t)]^{-1}}d\mathbf{B}_t,
    \end{cases}
\end{align*}
where $\mathbf{B}_t$ is the standard Brownian motion on $\mathbb{R}^{dN}$, for sampling from the Gibbs distribution $\mu_*^{(N)}$.

For time discretization of \eqref{eq:system}, we use the sequence of learning rates $\{\eta_k\}_{k\in\mathbb{N}}$ at each iteration and implement the forward discretization scheme from \citet{ahn2021efficient}. Let $\mathbf{X}_k= (X_k^i)_{i=1}^N\subseteq\mathcal{X}$ denote the $N$ particles at the $k$th update, and define $\mu_k = \mu_{\mathbf{X}_k}$ as the corresponding empirical distribution. Starting from $X_0^i\sim \mu_0$, the discretized MMFLD updates $\mathbf{X}_k$ as in Algorithm \ref{alg:mmfld}.

The dual algorithm is characterized as follows. Letting $\Tilde{\mu}_0 := \frac{1}{N}\sum_{j=1}^N\delta_{Y_0^i}$,  we know that $Y_{k+1}^i = \nabla\phi(X_{k+1})$ is the value at time $t = \eta_{k}$ of the stochastic process (written in differential form) 
\begin{equation}\label{eq:pre_weighted}
dY_t^{i} = - \nabla \frac{\delta F((\nabla\phi^*)_\sharp\Tilde{\mu}_0)}{\delta\mu}(\nabla \phi^*(Y_0^i))dt+\sqrt{2\lambda[\nabla^2\phi^*(Y_t^i)]^{-1}}dB_t^i, 
\end{equation}
with initial value given by the previous iteration $\nabla\phi(X_k^i)$. 

The forward discretization for mirror Langevin dynamics is known to have vanishing bias as $\eta_k\to 0$, while the standard Euler-Maruyama discretization does not \citep{ahn2021efficient,jiang2021mirror} without stronger assumptions on the mirror map $\phi$ \citep{li2022mirror}. In particular, \citet{li2022mirror} utilized a mean-square analysis \citep{li2019stochastic,li2021} to show that under a \emph{modified} self-concordance assumption, the standard Euler-Maruyama discretization does have vanishing bias under the weaker notion of Wasserstein error. We remark that the forward discretization from \citet{ahn2021efficient} (and thus also ours) can be viewed as discretizing the drift but not the diffusion. \citet{ahn2021efficient,jiang2021mirror} analyze the discretization where the corresponding SDE in step 5 of Algorithm \ref{alg:mmfld} is simulated exactly. We will follow their approach for simplicity of exposition as the bottleneck of the algorithm is the oracle complexity (number of queries to $\nabla \frac{\delta F}{\delta\mu}$), not tracking the pure diffusion. Nevertheless, our experiments (Section \ref{sec:exp}) indicate that a one-step discretization generally suffices.

\begin{algorithm}[t]
\begin{algorithmic}[1]
\REQUIRE{mirror map $\nabla\phi$, timestep $\eta_k$, max iterations $T$, number of particles $N$}
\STATE{Initialize $X_0^1,\cdots,X_0^N\sim\mu_0$}
\FOR{$k=0,\cdots,T-1$}
    \FOR{$i\in [N]$}
        \STATE{$Y_0^i \gets \nabla\phi(X_k^i) - \eta_k\,\nabla\frac{\delta F(\mu_k)}{\delta \mu}(X_k^i)$}
        \STATE{Use an Euler-Maruyama discretization to simulate the following diffusion for $t\in[0,\eta_k]$:}
        \begin{align*}
        dY_t^i = \sqrt{2\lambda\,\bigl[\nabla^2\phi^*(Y_t^i)\bigr]^{-1}}\,dB_t
        \end{align*}
        \STATE{$X_{k+1}^i \gets \nabla\phi^*(Y_{\eta_k}^i)$}
    \ENDFOR
    \STATE{$\mu_{k+1} \gets \frac{1}{N}\sum_{i=1}^N \delta_{X_{k+1}^i}$}
\ENDFOR
    \RETURN{$X_T^1,\dots, X_T^N$}
       \caption{Discretized MMFLD}\label{alg:mmfld}
\end{algorithmic}
\end{algorithm}


\subsection{Propagation of Chaos}\label{sec:chaosmfnn}
In the motivating example for MFLD, we focus on the supervised risk minimization problem of mean-field neural networks \citep{nitanda2022convex,suzuki2023meanfield}. Let $\mathcal{D}$ be the data space. Let $h(x,\cdot):\mathcal{D}\to\mathbb{R}^m$ be a function parameterized by $x\in\mathcal{X}$, representing a single neuron with weight $x$. The mean-field model is obtained by integrating $h(x,\cdot)$ with respect to a probability measure $\mu\in\mathcal{P}_2(\mathcal{X})$ over the parameter space, corresponding to the average output of infinitely many neurons distributed according to $\mu$: $h_\mu(\cdot) = \int h(x,\cdot)d\mu(x)$. Given training examples $\{(z_j,y_j)\}_{j=1}^n\subset\mathcal{D}\times\mathbb{R}$ and loss function $\ell:\mathbb{R}^m\times\mathcal{Y}\to\mathbb{R}$, the empirical risk of the mean-field neural network is
\begin{equation}\label{eq:F_0}
    \textstyle F_0(\mu) = \frac{1}{n}\sum_{j=1}^n\ell(h_\mu(z_j),y_j).
\end{equation}
We generalize the setting of \citet{nitanda2024improved}. They consider loss functions $\ell$ that are scalar-valued in both arguments. We show that their method straightforwardly extends to vector-valued arguments. 
\begin{assumption}\label{asmp:loss}
    $\ell(\cdot,y)$ is convex and smooth, and $h(X,z)$ has a finite second moment, e.g.,
    \begin{enumerate}
        \item[(i)] There exists $L>0$ such that for every $a,b\in\mathbb{R}^m,y\in\mathcal{Y}$, it holds that $\ell(b,y)\le \ell(a,y) + \langle\nabla\ell(a,y),b-a\rangle + \frac{L}{2}\|b-a\|^2$. 
        \item[(ii)] There exists $R>0$ such that for every $z\in\mathcal{D}$, it holds that $\mathbb{E}_{X\sim\mu_*}\|h(X,z)\|^2\le R^2$.
    \end{enumerate}
\end{assumption}
In the scalar setting, this has been verified in many natural settings, including neural networks with bounded activation functions and the logistic and squared losses \citep{nitanda2022convex,chizat2022meanfield,suzuki2023meanfield,nitanda2024improved}. By extending to the vector setting, this also covers density estimation and discrepancy minimization, e.g. see \citet[Appendix A]{suzuki2023meanfield}. Using this assumption, we have the following LSI constant-free particle approximation error concerning the objective gap.
\begin{theorem}\label{thm:poc}
    Under Assumptions \ref{asmp:rel_lip_and_smooth} and \ref{asmp:loss}, it holds that \[
    \frac{\lambda}{N}{\KL(\mu_*^{(N)}\sep\mu_*^{\otimes N})}\le \frac{1}{N}\mathcal{L}^{(N)}(\mu_*^{(N)}) - \mathcal{L}(\mu_*) \le \frac{LR^2}{2N}.
    \]
\end{theorem}
\begin{proof}[Proof (sketch)]
    By inspecting the arguments in \citet[Appendix A]{nitanda2024improved}, we notice that (i) the additional $L^2$ regularization term $\lambda' \int\|\cdot\|^2d\mu$ is not necessary, (ii) all scalar inputs to $\ell$ can suitably be replaced by vector values, and (iii) the results also covers the constrained setting $\mathcal{P}_2(\mathcal{X})$ by convexity of $\mathcal{X}$.
\end{proof}

\subsection{Convergence Analysis for Discretized MMFLD}

To utilize the one-step interpolation argument \citep{vempalawibisono} in our setting, we follow \citet{wibisono2019proximal,jiang2021mirror} and consider the discretized MMFLD as a \emph{weighted} dynamics (see Lemma \ref{lem:weighted}). The final ingredients we require are self-concordance and the uniform-in-$N$ mirror LSI.

\begin{assumption}[Self-concordance {\citep{gatmiry2022convergence}}]\label{asmp:self_concordance}
the conjugate mirror map $\phi^*$ satisfies the following: for every $x\in \operatorname{int}\mathcal{X}$ and $u\in\mathbb{R}^d$, it holds that \[
    |\nabla^3\phi^*(x)[u,u,u]| \le 2c_1\langle u, \nabla^2\phi(x)u\rangle^{3/2}.
    \]
Additionally, we let $c_2 := \lambda_{\min}(\nabla^2\phi)>0$.
\end{assumption}
Note that we can always ensure $\phi$ satisfies the last condition (regarding $c_2$) by adding a strongly convex regularization term $\frac{c_2}{2}\|\cdot\|_2^2$, at the expense of slowing down optimization due to a worse LSI constant.
Self-concordance is a standard assumption for interior point and mirror analyses \citep{nesterov1994interior,bubeck2015convex} and holds for classic examples such as the log-barrier function on a polytope \citep{gatmiry2022convergence}, barriers on convex sets defined by hyperbolic constraints \citep{Nar16}, or epigraphs of matrix norms \citep{nesterov1994interior}. Existing barriers can also be straightforwardly combined to provide explicit barriers for product spaces, e.g. Moreover, for any $d$-dimensional convex set, there always exists a self-concordant barrier with parameter $O(d)$ \citep{Nar16}. 

\begin{assumption}[Uniform-in-$N$ mirror LSI]\label{asmp:umlsi}
    There exists a constant $\CLSI> 0$ such that $\mu_*^{(N)}$ satisfies a mirror log-Sobolev inequality with constant $\CLSI$, that is for every $\mu^{(N)}$, it holds that
    \begin{align*}
\KL(\mu^{(N)}\sep\mu_*^{(N)}) \le \frac{1}{2\CLSI}\FI_\phi(\mu^{(N)}\sep\mu_*^{(N)}).
    \end{align*}
\end{assumption}
This assumption can be verified using Proposition \ref{prop:mirror_lsi}, for instance, via the uniform-in-$N$ LSI \citep{chewi2024uniform} and a strongly-convex mirror map.

With these tools, we can finally provide the following convergence guarantee for the discretized MMFLD.
\begin{theorem}\label{thm:main}
    Suppose Assumptions \ref{asmp:rel_lip_and_smooth}-\ref{asmp:umlsi} hold. Then Algorithm \ref{alg:mmfld} with step size $\eta$ satisfies \begin{align*}
    &\frac{1}{N}\mathcal{L}^{(N)}(\mu_k^{(N)}) - \mathcal{L}(\mu_*)\le \exp\left(-\CLSI \lambda\eta k\right)\left(\frac{1}{N}\mathcal{L}^{(N)}(\mu_0^{(N)}) - \frac{1}{N}\mathcal{L}^{(N)}(\mu_*^{(N)})\right) + \frac{LR^2}{2N} + \frac{\delta_\eta}{2\CLSI\lambda},    
    \end{align*}
    where $\delta_\eta := 2\eta M_2^4M(\eta M_1^2 + 2\lambda d)$, $D:= \max_{u,v}\|\nabla\phi(u)-\nabla\phi(v)\|_2$, and $M := \exp\left(\frac{2c_1D}{\sqrt{c_2}}\right)$.
\end{theorem}
The proof is given in Appendix \ref{app:thm_main}. This extends the discretization analysis of \citet{nitanda2024improved,nitanda2025}, which to the best of our knowledge constitute the sharpest known bounds for propagation of chaos of MFLD, to the mirror Langevin setting.

\subsection{Convergence Analysis with Stochastic Gradients}
In certain applications of MFLD, using stochastic gradients is desirable, e.g. for computational efficiency or for differential privacy \cite{gu2025private}. To our knowledge, the most closely related work is \cite{hsieh2018}, who analyze \emph{mirrored} Langevin dynamics with stochastic gradients, but this is a different setting than ours. 

We will use $g_k^i = g_k^i(\mathbf{X}_{0:k},\omega^i)$ to denote the stochastic approximation of $\nabla\frac{\delta F(\mu_k)}{\delta\mu}(\mathbf{X}_k^i)$, where $\omega_k = (\omega_k^1,\dots,\omega_k^N)$ is the randomness used to generate the stochastic gradient. Here $g_k$ can depend on the whole history $\mathbf{X}_{0:k}=(\mathbf{X}_0,\dots,\mathbf{X}_k)$. We will impose the following assumption.
\begin{assumption}\label{asmp:stoch_gradients}
    The following hold. 
    \begin{enumerate}
        \item[(i)] $\E_{\omega_k^i}[g_k^i\mid \mathbf{X}_{0:k}] = \nabla\frac{\delta F(\mu_k)}{\delta\mu}(\mathbf{X}_k^i)$.
        \item[(ii)] $\E_{\omega_k^i}[\|g_k^i - \nabla \frac{\delta F(\mu_k)}{\delta\mu}(\mathbf{X}_k^i)\|^2]\le \sigma^2$ for some $\sigma>0$.
    \end{enumerate}
\end{assumption}
(ii) can be verified with a bound of $\sigma^2 = {R'}^2/B$ via $\sup_x\|\nabla\frac{\delta\ell(a,y)}{\delta\mu}(x)\|\le R'$ for all $a\in\mathbb{R}^d,y\in\mathcal{Y}$ and batched gradients of size $B$; see \cite{suzuki2023meanfield} for more details. By Assumption~\ref{asmp:self_concordance}, (ii) also implies $\E_{\omega_k^i}[\|g_k^i - \nabla \frac{\delta F(\mu_k)}{\delta\mu}(\mathbf{X}_k^i)\|^2_{[\nabla^2\phi]^{-1}}]\le \frac{\sigma^2}{c_2}$. 

\begin{theorem}\label{thm:stoch_grad}
    Suppose Assumptions \ref{asmp:rel_lip_and_smooth}-\ref{asmp:stoch_gradients} hold. Then Algorithm \ref{alg:mmfld} using stochastic gradients and step size $\eta$ satisfies 
    \begin{align*}
    &\frac{1}{N}\mathcal{L}^{(N)}(\mu_k^{(N)}) - \mathcal{L}(\mu_*)\le \exp\left(-\CLSI \lambda\eta k\right)\left(\frac{1}{N}\mathcal{L}^{(N)}(\mu_0^{(N)}) - \frac{1}{N}\mathcal{L}^{(N)}(\mu_*^{(N)})\right)+ \frac{LR^2}{2N} + \frac{2\delta_\eta + \sigma^2/c_2}{2\CLSI\lambda}.    
    \end{align*}
\end{theorem}
Here, $\delta_\eta$ is defined the same as in Theorem \ref{thm:main}, and the proof is provided in Appendix \ref{app:thm_stoch_grad}. 



\section{Numerical Experiments}\label{sec:exp}

\begin{figure}
    \centering
    \includegraphics[width=0.79\linewidth]{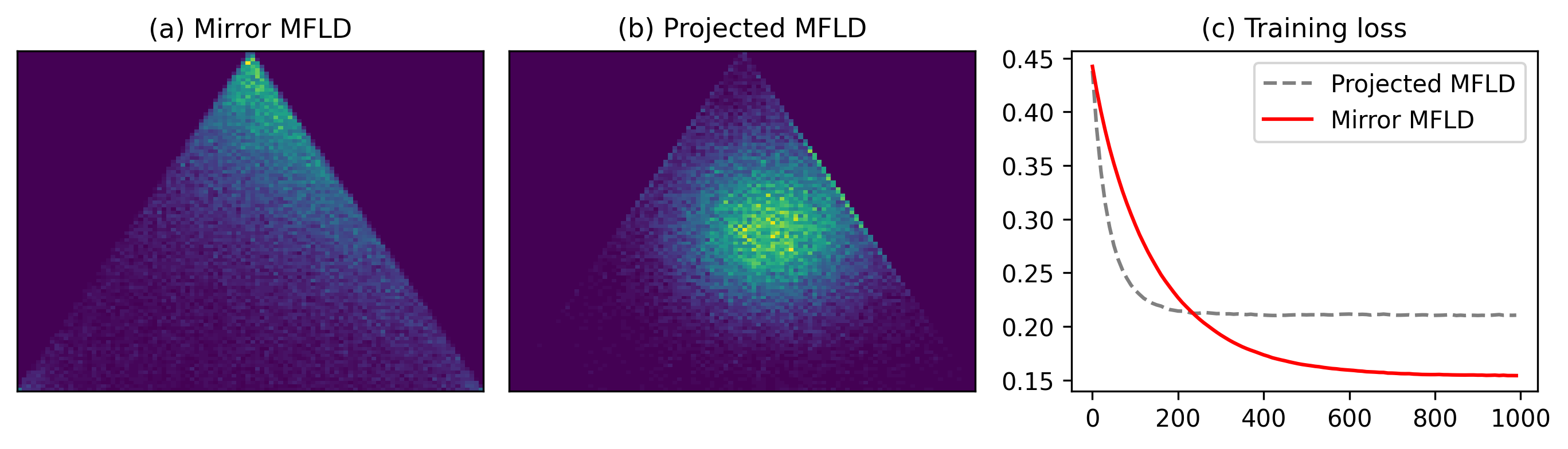}
    \includegraphics[width=0.79\linewidth]{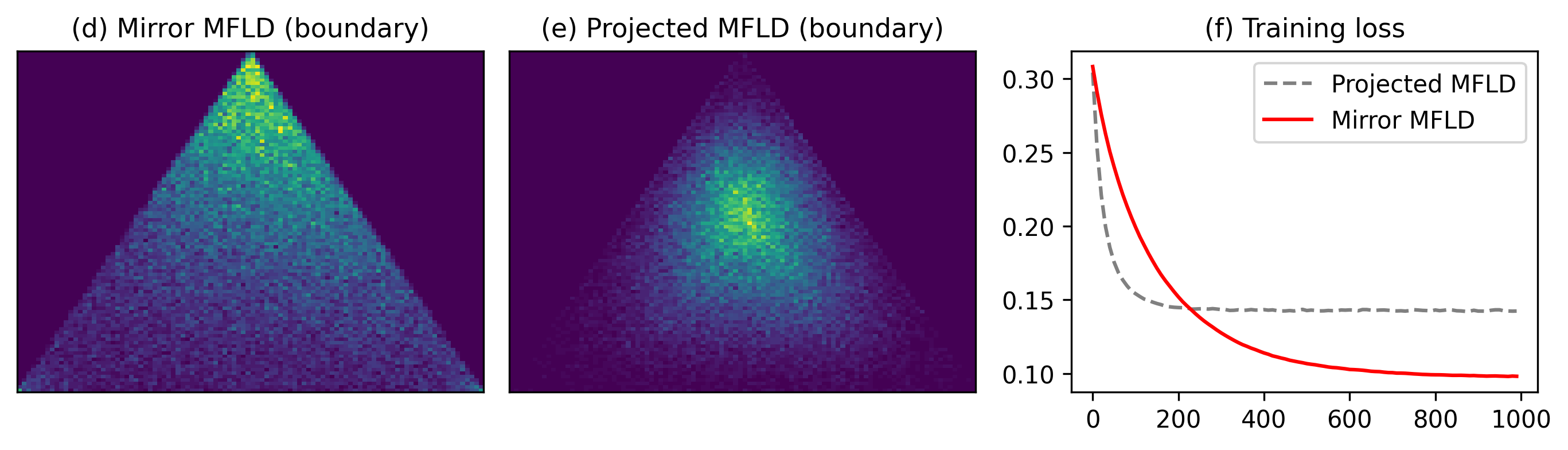}

    \caption{(a)-(c) MMFLD v.s. projected MFLD for optimizing a nonlinear mean-matching objective over the unit simplex; (d)-(f) with a boundary potential preventing mass accumulation.}
    \label{fig1}
    \vspace{-0.1in}
\end{figure}

In this section, we conduct experiments demonstrating the utility of (discretized) MMFLD. In Appendix~\ref{app:additional_experiments}, we provide additional experiments in high dimensions.

\subsection{Optimization over the Simplex}

We consider optimization on the unit simplex $\Delta_d = \{x\in\mathbb{R}^d\,|\,\sum_{i=1}^d x_i = 1,\; x_i\ge 0\}$. A natural choice for the mirror map is the entropy function $\phi(x) = \sum_i x_i\log x_i$, whose dual is given by $\phi^*(y) = \log\sum_i\exp y_i$; see \citet{Beck03} for details. We set $d=3$ for easy visualization. We aim to minimize the nonlinear functional
\begin{align*}
\textstyle F(\mu) = \left\lVert\int_{\Delta_3} x d\mu(x) - q\right\rVert^2 + \beta \int_{\Delta_3} \sum_{i=1}^3 \log\frac{1}{x_i} d\mu(x).
\end{align*}
The first term is a mean-matching score for a fixed target $q\in\mathbb{R}^3$, while the second term is a potential barrier preventing collapse towards the boundary $\partial \Delta_3$. As a baseline, we also implement projected MFLD, where particles are projected back onto the simplex after each update \eqref{eq:mfld_sde}. The diffusion step of MMFLD is simulated in one step, resulting in similar runtimes. Both algorithms are run using $50$K particles with $\eta=3\times 10^{-3}$, $\lambda=0.1$.

Figure \ref{fig1}a-c show the loss curves and final particle distributions when $\beta=0$. In particular, MFLD exhibits accumulation of mass along the boundary characteristic of projection methods (Figure \ref{fig1}b), which leads to excess error or overfitting. To avoid this trivial failure case, we also add a small barrier preventing convergence to the boundary by setting $\beta=10^{-4}$ (Figure \ref{fig1}d-f). We observe a drastic effect on projected MFLD as particles are repelled from the boundary; in contrast, MMFLD maintains a more even distribution as the mirror map already ensures the particles are well-behaved near the boundary. In both experiments, MFLD initially converges faster but MMFLD is able to achieve significantly smaller loss, thus demonstrating higher stability and optimality.

\subsection{Optimization over the Spectraplex}

We next study optimization over the \emph{spectraplex} $\mathcal{X}= \{\Sigma \in \mathbb{S}_+^d\mid\Tr(\Sigma) = 1\}$, where $\mathbb{S}_+^d$ is the set of $d\times d$ positive semidefinite matrices. Let $\Sigma^*\in\mathcal{X}$ be a target density matrix. Define the following mean-field functional: \[
\textstyle F(\mu) = \frac{1}{2}\left\|\mathbb{E}_\mu[\Sigma] - \Sigma^*\right\|_\mathsf{F}^2 + \frac{1}{2} \gamma \int \|\Sigma\|_{\mathsf{F}}^2\, d\mu(\Sigma)
\]

\begin{figure}
\centering
\includegraphics[width=0.3\linewidth]{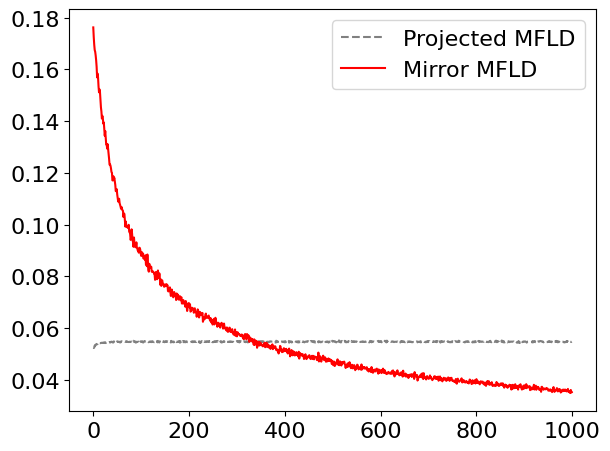}
\caption{Optimization over the spectraplex with von Neumann mirror map.}
\label{fig:psd}
\vspace{-0.1in}
\end{figure}

for $\gamma>0$. For projected MFLD, the projection step amounts to taking the eigendecomposition, and renormalizing the eigenvalues to sum to $1$. For MMFLD, we use the von Neumann mirror map $\phi(\Sigma) = \Tr(\Sigma\log \Sigma - \Sigma)$.

In the experiment, we take $d = 10$, $\gamma = 0.02$, $\eta = 0.3$, $\lambda = 0.1$, and $N = 2048$, and choose a $\Sigma^*\sim \mathrm{Unif}(\mathcal{X})$. Again, the diffusion is simulated using one step. We present the results in Figure \ref{fig:psd}. Observe that projected MFLD makes very little progress, compared to MMFLD. We hypothesize this is due to the projection step roughly canceling out all of the progress we make in the iteration.

\subsection{Classification with Mean-field Network}\label{sec:mfnn_classification}

We consider the more practical problem of optimizing a norm-constrained neural network for a classification task. Norm constraints are used to stabilize training and improve generalization by preventing overly large parameter values or overfitting. The domain is the $d$-dimensional unit ball with ball barrier $\phi(z)\propto -\log(1-\lVert z\rVert_2^2)$. For the score function, we use a two-layer ReLU network $f(x)=\frac{1}{N} \sum_{i=1}^N \text{relu}(\langle w_i,x\rangle)$ with the constraint $\lVert w_i\rVert_2\le 1$. The prediction is given by the sigmoid transformation $\sigma(f(x))$ with cross-entropy loss. We use XOR data with Gaussian noise, so that the nonlinearity is necessary to learn the decision boundaries. We implement MMFLD and projected MFLD to learn $\{w_i\}_{i=1}^N$. Hyperparameters are set as follows: $N=512$, $d=2$, $\eta=0.1$, $\lambda=10^{-3}$, $T=100$. Again, the diffusion is simulated in one step.

\begin{figure}
    \centering
    \includegraphics[width=0.79\linewidth]{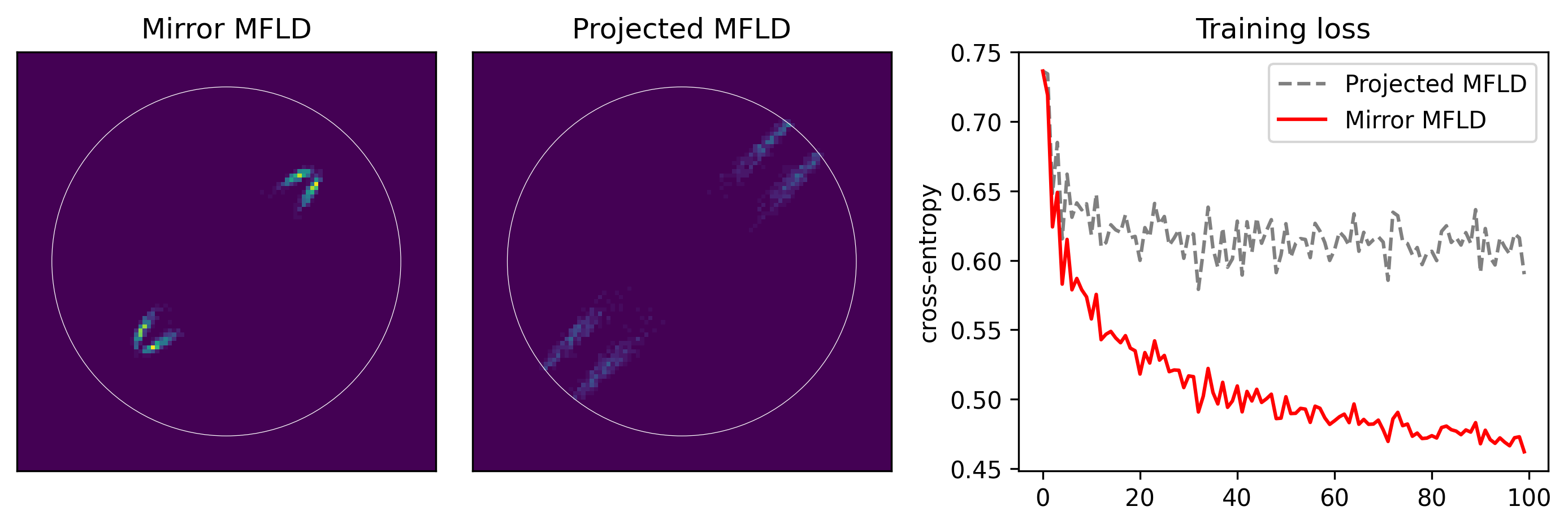}

    \caption{(Left) MMFLD v.s. projected MFLD for optimizing a norm-constrained two-layer ReLU neural network for classification; (right) neuron distribution along XOR directions.}
    \label{figxor}
    \vspace{-0.1in}
\end{figure}

We plot loss curves and visualize the neuron distribution in the XOR directions in Figure \ref{figxor}. MMFLD decreases the loss at a significantly faster rate than projected MFLD and continues to decrease error after 100 epochs, while projected MFLD stagnates after around 30-50 epochs. The same behavior is observed for the test curves. Moreover, the density plots show that MMFLD is able to align the neurons tightly in the XOR directions, while MFLD scatters the neurons over a wider region and saturates the norm constraint, increasing uncertainty in the prediction and generalization error.

\section{Conclusion}
We introduced the mirror mean-field Langevin dynamics to solve distributional optimization problems constrained to a convex subset of $\mathbb{R}^d$. Moreover, we obtained linear convergence rates for the continuous-time flow, space- and time-discretized, and stochastic gradient settings. An important direction for future work is to generalize to the mirror Poincar\'e setting as in \citet{chewi2020mld} since it is more natural from a geometric viewpoint \citep{chewi2023logconcavesampling}; this will require studying the mean-field Langevin dynamics under a Poincar\'e inequality as well \citep{nitanda2022convex}.

\section*{Acknowledgements}
This work was partially done when AG was at Boston University and JK was at the University of Tokyo and RIKEN AIP.

\bibliography{references}
\bibliographystyle{abbrvnat}

\newpage

\appendix

\section{Related Work}\label{app:related}
\paragraph{MLD and constrained sampling.} For the problem of constrained sampling, the mirrored Langevin dynamics was first proposed by \citet{hsieh2018}. Then, \citet{zhang2020wasserstein} introduced the more general mirror Langevin diffusion with Hessian-dependent covariance and its Monte Carlo discretization, the mirror Langevin algorithm (MLA). In continuous time, \citet{chewi2020mld} show exponential convergence of MLD under a mirror Poincar\'{e} inequality relative to the modified metric. In particular, this implies convergence of the Newton-Langevin diffusion when the mirror map equals the (strictly convex) potential $f$.

The convergence of the discretized algorithm is more difficult to analyze compared to the ordinary Langevin dynamics. \citet{zhang2020wasserstein} showed that the MLA iterates contract to a Wasserstein ball around the target density, but with a potentially non-vanishing bias. \citet{ahn2021efficient} proposed a different discretization of MLD which achieves vanishing bias, but assumed an exact computation of the Hessian Riemannian throughout the diffusion process.
On the other hand, \citet{li2022mirror} showed vanishing bias of MLA under more restrictive assumptions on the mirror map based on \citet{li2019stochastic}; and \citet{gatmiry2022convergence} utilized an additional second-order self-concordance assumption (cf. Assumption \ref{asmp:self_concordance}) and extensive Riemannian geometry for an exact discretization analysis on Hessian manifolds. Furthermore, Metropolis-adjusted (high accuracy) variants of MLA have been proposed to obtain algorithms whose stationary distributions are unbiased \citep{srinivasan2024fast,srinivasan2024high}. A parallel line of works considers a non-Euclidean generalization of the proximal sampler \cite{gopi2023algorithmic,gu2026functional}.

\paragraph{MFLD and propagation of chaos.} Initial mean-field analyses of two-layer neural networks were studied independently by several groups \citep{mei2018meanfield,chizat2018global,sirignano2018mfnn,hu2021mfld}, showing either weak convergence or linear convergence under restrictive assumptions, e.g. the mean-field interaction is small or the optimization is over-regularized \citep{eberle2019quantitative}. Under the Langevin setting, the linear convergence of the MFLD (under mild conditions) was first shown in \citet{nitanda2022convex,chizat2022meanfield}, leveraging the log-Sobolev inequality of the proximal Gibbs measure. Whereas \citet{nitanda2022convex} studied the time-discretization and offered a primal-dual analysis, \citet{chizat2022meanfield} studied the annealed dynamics where $\lambda_t\to 0$ in \eqref{eq:mfld_sde_intro}. 

Following works focused on proving propagation of chaos bounds for the corresponding finite particle dynamics, with the noisy gradient descent of finite-width neural networks as a prominent application. While earlier arguments indicated the possibility of error blowing up exponentially in time \citep{mei2018meanfield}, uniform-in-time guarantees were obtained in the case of super-quadratic regularization \citep{suzuki2023uniformintime} and for more natural quadratic regularization \citep{chen2023uniform,suzuki2023meanfield,kook2024} via a uniform-in-$N$ LSI \citep{wang2024,chewi2024uniform}. Furthermore, this error was shown to be independent of the LSI constant in \citet{nitanda2024improved} and the corresponding convergence analysis was further refined by \citet{chewi2024uniform,nitanda2025}, respectively relaxing quadratic regularization to strongly convex regularization and improving the propagation of chaos results using a directional LSI. Also, improved sample complexity results have been established for certain structured classification problems \citep{suzuki2023feature,mousavi-hosseini2025learning}.

\section{Mirror LSI}\label{app:mlsi}
Here, we provide a simple and general method to establish a mirror LSI.
\begin{proposition}[\citet{Daaloul25}]\label{prop:mirror_lsi}
    If $\nu\in\mathcal{P}(\mathcal{X})$ satisfies a LSI with constant $\CLSI$ and $\phi$ is $\alpha$-strongly convex, then $\nu$ satisfies a mirror LSI with constant $\CLSI/\alpha$.
\end{proposition}
For completeness, we provide the proof from \citep{Daaloul25}.
\begin{proof}
    We have \begin{align*}
        \FI_\phi(\mu\sep\nu) &= \int\left\|\nabla\log \frac{\mu}{\nu}\right\|_{[\nabla^2\phi \circ \nabla \phi^*]}^2d\nu\\
        &\ge \frac{1}{\alpha}\int \left\|\nabla\log \frac{\mu}{\nu}\right\|^2d\nu\\
        &\ge \frac{2\CLSI}{\alpha}\KL(\mu\sep\nu).
    \end{align*}
    This concludes the proof.
\end{proof}

As another example, we provide an explicit family of (one-dimensional) distributions that satisfies a mirror LSI, due to \citet{Furioli19}.
\begin{proposition}[{\citet[Theorem 3]{Furioli19}}]
Define the densities \[
\nu_{m,\lambda}\propto (1-x)^{-1+\frac{1-m}{\lambda}}(1+x)^{-1+\frac{1+m}{\lambda}}
\]
for $\lambda>0$ and $m\in(-1,1)$. Let the barrier function be \[
\phi(x) = \frac{1}{2}[(1+x)\log(1+x)+(1-x)\log(1-x)].
\]
Suppose $1 - \frac{\lambda}{2}>0$ if $m = 0$ and $1-\frac{\lambda}{2}\ge |m|$, otherwise. Then, $\nu_{m,\lambda}$ satisfies a mirror LSI with constant \[
\CLSI = \frac{\lambda}{2}\left(1-\frac{\lambda}{2}+\sqrt{\left(1-\frac{\lambda}{2}\right)^2-m^2}\right)^{-1}.
\]
\end{proposition}

\subsection{Properties of the mirror LSI}
\begin{lemma}[Stability under bounded perturbation, {\citet[Lemma 2]{jiang2021mirror}}]
    Suppose $\nu$ satisfies a mirror LSI with constant $\CLSI$, then if $c_1 \le \frac{d\mu}{d\nu}\le c_2$, then $\mu$ satisfies a mirror LSI with constant $\frac{c_1}{c_2}\CLSI$.
\end{lemma}
We remark that \citet{jiang2021mirror} made a minor typo and stated the flipped result is the mirror LSI constant. For completeness, we provide the proof from \cite{jiang2021mirror}.
\begin{proof}
    Using the variational principle of entropy, for $\mu\ll \nu$, we have \begin{align}
    \mathsf{Ent}_\mu(f^2) \le \left\|\frac{d\mu}{d\nu}\right\|_\infty\mathsf{Ent}_\nu(f^2) \le c_2\cdot \mathsf{Ent}_\nu(f^2).   \label{eq:stab1}     
    \end{align}
    Next, we have \begin{align}
        \frac{2}{\CLSI}\int\|\nabla f(x)\|^2_{[\nabla^2\phi(x)]^{-1}}d\nu &= \frac{2}{\CLSI}\int\|\nabla f(x)\|^2_{[\nabla^2\phi(x)]^{-1}}\frac{d\nu}{d\mu}d\mu\notag\\
        &\le \frac{2}{c_1}\int\|\nabla f(x)\|^2_{[\nabla^2\phi(x)]^{-1}}d\mu.\label{eq:stab2}
    \end{align}
    Combining \eqref{eq:stab1} and \eqref{eq:stab2}, we have \[
    \frac{2c_2}{c_1\CLSI}\int\|\nabla f(x)\|^2_{[\nabla^2\phi(x)]^{-1}}d\mu \ge \mathsf{Ent}_\mu(f^2),
    \]
    as desired.
\end{proof}

\begin{lemma}[Convolution]
    Suppose $\mu_1, \mu_2$ respectively satisfy mirror LSI with constant $\CLSI^{(1)}, \CLSI^{(2)}$. Then, their convolution $\mu := \mu_1 * \mu_2$ satisfies a mirror LSI with constant $\left(\frac1{\CLSI^{(1)}}+\frac1{\CLSI^{(2)}}\right)^{-1}$.
\end{lemma}
\begin{proof}
    Let $X\sim \mu_1, Y\sim \mu_2$, and let $f:\R^d\to\R$. Using subadditivity of entropy, \begin{align*}
        \Ent_\mu(f^2)&\le \E[f(X+Y)^2\mid X] + \E[f(X+Y)^2\mid Y]\\
        &\le \left(\frac{2}{\CLSI^{(1)}}+\frac{2}{\CLSI^{(2)}}\right)\E[\|\nabla f\|_{[\nabla^2\phi]^{-1}}^2],
    \end{align*}
which concludes the proof.
\end{proof}

\begin{lemma}[Tensorization under separable mirror map]
    Let $\{\mu_i\}_{i\in[m]}$ be probability measures on $\mathbb{R}^{d_i}$, $\{\phi_i:\mathbb{R}^{d_i}\to \mathbb{R}\}_{i\in[m]}$ be mirror maps that satisfy the conditions in Section \ref{sec:prelim}, and $\mu_i$ satisfies a mirror LSI (with respect to $\phi_i$) with constant $\CLSI^{(i)}$, for every $i \in[n]$. Define $\mu := \bigotimes_{i=1}^m \mu_i$ and the separable map $\phi(x) := \sum_{i=1}^m \phi_i(x_i)$, which satisfies \[
    \nabla^2\phi(x) = \mathsf{diag}(\nabla^2\phi_1(x_1),\dots, \nabla^2\phi_m(x_m)).
    \]
    Then $\mu$ satisfies a mirror LSI (with respect to $\phi$) with constant $\min_{i\in[m]} \CLSI^{(i)}$.
\end{lemma}
\begin{proof}
Let $X_i\sim \mu_i$ be independent for $i\in[m]$. Let $X_{-i}$ denote all coordinates except $i$. For any $f\ge 0$, we have \begin{align*}
        \mathsf{Ent}_\mu(f) &\le \sum_{i=1}^m \mathbb{E}\left[\mathsf{Ent}_{\mu_i}(f(X_1,\dots, X_m)\mid X_{-i}\right]\\
        &\le \sum_{i=1}^m \frac{2}{\CLSI^{(i)}}\mathbb{E}\left[\|\nabla_{i} f(X_1,\dots, X_m)\|_{[\nabla_{i}^2\phi_i(X_i)]^{-1}}^2\mid X_{-i}\right]\\
        &\le \frac{2}{\min_{i\in[m]}\CLSI^{(i)}}\sum_{i= 1}^m\mathbb{E}\left[\|\nabla_{i} f(X_1,\dots, X_m)\|_{[\nabla_{i}^2\phi_i(X_i)]^{-1}}^2\mid X_{-i}\right]\\
        &= \frac{2}{\min_{i\in[m]}\CLSI^{(i)}}\mathbb{E}\left[\|\nabla f(X_1,\dots, X_m)\|^2_{[\nabla^2\phi(X_1,\dots, X_m)]^{-1}}\right],
    \end{align*}
    where the first inequality uses tensorization of entropy, the second uses the mirror LSI, and the equality follows from separability of the mirror map. The claim follows.
\end{proof}

\section{Deferred Proofs and Results}\label{app:proofs}
\subsection{Properties of $\phi$}
\begin{lemma}[Properties of $\phi$]\label{lem:properties_of_phi}
The following hold:
\begin{enumerate}
    \item[(i)] $\nabla \phi(\mathcal{X})=\mathbb{R}^d$, the gradient map $\nabla\phi:\mathcal{X}\to\mathbb{R}^d$ is bijective and $\nabla^2\phi(x)\succ 0$ for all $x\in\mathcal{X}$.
    \item[(ii)] $\nabla\phi^*(y) = \arg\max_{x\in\mathcal{X}}\langle x,y\rangle-\phi(x)$.
    \item[(iii)]  $\nabla\phi^* = (\nabla\phi)^{-1}$, so $\nabla\phi(\nabla\phi^*(y)) = y$ for all $y\in\mathbb{R}^d$, and $\nabla^2\phi(x) =\nabla^2\phi^*(\nabla\phi(x))^{-1}$ for all $x\in\mathcal{X}$.
\end{enumerate}
\end{lemma}

\subsection{Proof of Theorem \ref{thm:mld_convergence}}\label{sec:b1}

\begin{proof}
Using the Fokker-Planck PDE of \eqref{eq:mld_primal}, integration by parts, and Assumption \ref{asmp:lsi}, we have \begin{align*}
        \frac{d}{dt}\KL(\mu_t\sep\mu) &= \int \frac{\partial\mu_t}{\partial t}(x)\log\frac{\mu_t}{\mu
        }(x)dx\\
        &= \lambda\int \nabla \cdot \left(\mu_t(x)[\nabla^2\phi(x)]^{-1}\nabla \log \frac{\mu_t}{\mu}(x)\right)\log \frac{\mu_t}{\mu}(x)dx\\
        &= -\lambda\int \left\langle \nabla \log \frac{\mu_t}{\mu}, [\nabla^2\phi]^{-1}\nabla\log \frac{\mu_t}{\mu}\right\rangle d\mu_t\\
        &= -\lambda  \FI(\mu_t\sep\mu)\\
        &\le -2\lambda\CLSI\cdot \KL(\mu_t\sep\mu).
    \end{align*}
    The claim follows from Gr\"onwall's inequality.    
\end{proof}

\subsection{Proof of Theorem \ref{thm:unique}}\label{app:proofunique}

\begin{proof}
This follows from an analogous argument to that for the existence and uniqueness of the mirror Langevin dynamics \citep{zhang2020wasserstein,chewi2020mld}. Well-posedness of \eqref{eq:mmfld_L} in the dual space has been shown, for instance, via \citet{rockner2021well}. This implies well-posedness in primal space via Lemma \ref{lem:properties_of_phi}. If \eqref{eq:mmfld_L} admits a minimizer, its uniqueness follows from strict convexity of entropy. The property $\mu_* \propto \exp\left(-\frac{1}{\lambda}\frac{\delta F(\mu_*)}{\delta\mu}\right)$ follows from the first-order optimality condition, which says that $\frac{\delta F(\mu_*)}{\delta\mu} + \lambda \log \mu_* $ is a constant $\mu_*$-almost everywhere. Then due to the entropy term, $\mu_*$ has positive density everywhere. Further details can be found in e.g. \citet{mei2018meanfield,chizat2022meanfield}.
\end{proof}

\subsection{Entropy sandwich}

\begin{lemma}[Entropy sandwich]\label{lem:entropy_sandwich}
    Under Assumption \ref{asmp:smoothness}, let $\mu_*$ be the unique minimizer of \eqref{eq:mmfld_L}. For $\mu\in\mathcal{P}_2(\mathcal{X})$, let $\hat{\mu}$ be the proximal Gibbs measure associated to $\mu$. The following hold.
    \begin{enumerate}
        \item[(i)] In the sense of functionals on the space of probability density functions, \[
        \frac{\delta\mathcal{L}(\mu)}{\delta\mu} = \lambda \frac{\delta}{\delta\mu'}{\KL}(\mu'\sep\hat{\mu})\bigg|_{\mu'=\mu} = \lambda \log\frac{\mu}{\hat{\mu}}.
        \]
        \item[(ii)] For any $\mu,\mu'\in\mathcal{P}_2(\mathcal{X})$, it holds \[
        \mathcal{L}(\mu) + \int \frac{\delta\mathcal{L}(\mu)}{\delta\mu}d(\mu'-\mu) + \lambda{\KL(\mu'\sep\mu)}\le \mathcal{L}(\mu').
        \]
        \item[(iii)] For any $\mu\in\mathcal{P}_2(\mathcal{X})$,   it holds \[
    \lambda{\KL(\mu\sep\mu_*)}\le \mathcal{L}(\mu) - \mathcal{L}
(\mu_*)\le \lambda{\KL(\mu\sep\hat{\mu})}.
\]
    \end{enumerate}
\end{lemma}
\begin{proof}
    It is straightforward to see that the proof of \citet[Proposition 1]{nitanda2022convex} still holds over $\mathcal{P}_2(\mathcal{X})$ by convexity of $\mathcal{X}$.
\end{proof}

\subsection{Proof of Theorem \ref{thm:mmfld_conti}}\label{app:proofconti}

\begin{proof}
We combine Section \ref{sec:b1} and the entropy sandwich inequality in a straightforward manner. From the Fokker-Planck PDE \eqref{eq:mmfld_pde} and applying integration by parts, we have by Assumption \ref{asmp:ULSI} and Lemma \ref{lem:entropy_sandwich} that
    \begin{align*}
        \frac{d}{dt}(\mathcal{L}(\mu_t) - \mathcal{L}(\mu)) &= \int \frac{\delta \mathcal{L}(\mu_t)}{\delta\mu}(x)\frac{\partial\mu_t}{\partial t}(x)dx\\
        &= \lambda \int \frac{\delta\mathcal{L}(\mu_t)}{\delta\mu}(x)\nabla \cdot \left(\mu_t(x)[\nabla^2\phi(x)]^{-1}\nabla \log \frac{\mu_t}{\hat{\mu}_t}(x)\right)dx\\
        &= -\lambda\int \left\langle \nabla\frac{\delta\mathcal{L}(\mu_t)}{\delta\mu},[\nabla^2\phi]^{-1}\nabla \log \frac{\mu_t}{\hat{\mu}_t} \right\rangle d\mu_t\\
        &=  -\lambda^2\int \left\langle \nabla \log \frac{\mu_t}{\hat{\mu}_t} ,[\nabla^2\phi]^{-1}\nabla \log \frac{\mu_t}{\hat{\mu}_t} \right\rangle d\mu_t\\
        &= - \lambda^2{\FI(\mu_t\sep\hat{\mu}_t)}\\
        &\le -2\lambda^2\CLSI\cdot {\KL(\mu_t\sep\hat{\mu}_t})\\
        &\le -2\lambda\CLSI(\mathcal{L}(\mu_t) - \mathcal{L}(\mu_*)).
    \end{align*}
    The claim follows from Gr\"onwall's inequality.
\end{proof}

\subsection{Preliminaries for convergence results}

In this section, we collect additional results used in the proof of Theorem \ref{thm:main}. First, we introduce the characterization of the discretized MMFLD as a weighted dynamics. Using $X_t^i = \nabla\phi^*(Y_t^i)$ and $\mu_t = (\nabla\phi^*)_\sharp \Tilde{\mu}_t = \frac{1}{N}\sum_{j=1}^N\delta_{X_t^i}$, we can show the following.
\begin{lemma}\label{lem:weighted}
\eqref{eq:pre_weighted} is equivalent to the primal dynamics 
\begin{equation}\label{eq:weighted}
    dX_t^i = \left(\nabla \cdot G_t^i(X_t^i) - G_t^i(X_t^i)\nabla \frac{\delta F(\mu_t)}{\delta\mu}(X_t^i) + \pi^i\right)dt + \sqrt{2\lambda G_t^i(X_t)}dB_t^i,
\end{equation}
with \emph{shifted} covariance $G_t^i$ and drift $\pi^i$ defined by \begin{align*}
G_t^i(X_t^i) &= [\nabla^2\phi(X_t^i)]^{-1},\\
        \pi^i &= G_t^i(X_t^i)\left(\nabla \frac{\delta F(\mu_t)}{\delta\mu}(X_t^i) - \nabla \frac{\delta F(\mu_0)}{\delta\mu}(X_0^i)\right).
    \end{align*}
\end{lemma}
\begin{proof}
    This follows from It\^o's lemma, e.g. see \citet[Appendix C]{jiang2021mirror}.
\end{proof}

The following result generalizes Lemma 4 of \citet{jiang2021mirror}.

\begin{lemma}\label{lem:jiang4}
    Under Assumptions \ref{asmp:rel_lip_and_smooth} and \ref{asmp:self_concordance}, and consider the updates $X_t^i = \nabla\phi^*(Y_t^i)$ following \eqref{eq:pre_weighted}, it holds \[
    \frac{1}{M}[\nabla^2\phi(X_t^i)]^{-1}\preceq [\nabla^2 \phi(X_t^i)]^{-1}\nabla^2\phi(X_0^i)[\nabla^2\phi(X_t^i)] \preceq M [\nabla^2\phi(X_t^i)]^{-1}
    \]
    in expectation with respect to the Brownian motion for \[
    M = \frac{1-\exp(-\frac{1}{16c_1^2t})}{(1-c_1(tM_1 + 2\sqrt{td}))^2} + \exp\left(-\frac{1}{16c_1^2t} + \frac{2c_1D}{\sqrt{c_2}}\right)
    \]
    if $t \le \min\left\{\frac{1}{2c_1M_1}, \frac{1}{16c_1^2d}\right\}$, and deterministically bounded by $M = \exp\left(\frac{2c_1D}{\sqrt{c_2}}\right)$. We use the convention $M = 1$ when $c_1 = 0$ and $D = \infty$. 
\end{lemma}
\begin{proof}
    By Lemma \ref{lem:properties_of_phi}, this implies from \citet{nesterov1994interior} that for $\|Y_t^i - Y_0^i\|_{\nabla^2\phi^*(Y_0^i)}\le \frac{1}{c_1}$, it holds
    \begin{align*}
    (1-c_1\|Y_t^i - Y_0^i\|_{\nabla^2\phi^*(Y_0^i)})^2\nabla^2\phi^*(Y_0^i)\preceq \nabla^2\phi^*(Y_t^i)\\
    \preceq \frac{1}{(1-c_1\|Y_t^i - Y_0^i\|_{\nabla^2\phi^*(Y_0^i)})^2}\nabla^2\phi^*(Y_0^i).
    \end{align*}
Since
\begin{align*}
\nabla\phi(X_t^i)-\nabla\phi(X_0^i) = Y_t^i - Y_0^i = -t\nabla \frac{\delta F(\mu_0)}{\delta\mu}(X_0^i) + \sqrt{2\lambda t\nabla^2\phi(X_0^i)}\cdot Z_0^i
\end{align*}
we have \begin{align*}
    \|Y_t^i-Y_t^0\|_{\nabla^2\phi^*(Y_0^i)} &= \left\|-t \nabla \frac{\delta F(\mu_0)}{\delta\mu}(X_0^i) + \sqrt{2\lambda t\nabla^2\phi(X_0^i)}Z_0^i\right\|^2_{[\nabla^2\phi(X_0^i)^{-1}]}\\
    &= t^2 \left\|\nabla\frac{\delta F(\mu_0)}{\delta \mu}(X_0^i)\right\|^2_{[\nabla^2\phi(X_0^i)]^{-1}} + 2\lambda t \|Z_0^i\|_2^2\\
    &\le t^2M_1^2  + 2\lambda t\|Z_0^i\|_2^2.
\end{align*}
Using $\chi^2$ concentration from \citet{Laurent2000AdaptiveEO}, we have
\begin{align*}
\Pr[\|Z_0^i\|_2^2\ge (\sqrt{d} + \sqrt{\delta})^2]\le \exp(-\delta)
\end{align*}
for $t\le \min \left\{\frac{1}{2c_1M_1}, \frac{1}{16c_1^2d}\right\}$. With probability at least $1 - \exp(-d)\ge 1 - \exp\left(-\frac{1}{16c_1^2d}\right)$ over the draw of $Z_0^i$, it follows that
\begin{align*}
\|Y_t^i - Y_0^i\|_{\nabla^2\phi^*(Y_0^i)}\le tM_1 + 2\sqrt{td}< \frac{1}{c_1}.
\end{align*}
Thus, we have
\begin{align*}
(1-c_1(tM_1 + 2\sqrt{\lambda td}))^2I_d\preceq \nabla^2\phi(X_0^i)^{1/2}[\nabla^2\phi(X_t^i)]^{-1}\nabla^2\phi(X_0)^{1/2}
\preceq \frac{1}{(1-c_1(tM_1 + 2\sqrt{\lambda td}))^2}I_d.
\end{align*}
For the remaining probability $\exp\left(-\frac{1}{16c_1^2d}\right)$, consider the function
\begin{align*}
g_i(s) := \langle u, \nabla^2\phi^*(Y_0^i + s(Y_t^i - Y_0^i))u\rangle =: \nabla^2\phi^*(Y_s^i)[u,u].
\end{align*}
From self-concordance, we have \begin{align*}
    |g_i'(s)| &= |\nabla^3\phi^*(Y_s^i)[u,u,Y_t^i-Y_0^i]\\
    &\le 2c_1\|Y_t^i-Y_0^i\|_{\nabla^2\phi^*(Y_s)}\|u\|_{\nabla^2\phi^*(Y_s^i)}\\
    &= 2c_1\|Y_t^i-Y_0^i\|_{\nabla^2\phi^*(Y_s)}\cdot g_i(s)\\
    &\le \frac{2c_1}{\sqrt{c_2}}\|Y_t^i-Y_0^i\|_{2}\cdot g_i(s)\\
    &\le \frac{2c_1}{\sqrt{c_2}}D\cdot g_i(s),
\end{align*}
which implies $|{\log g(1)} - \log g(0)|\le \frac{2c_1D}{\sqrt{c_2}}$. Then, we deduce \begin{align*}
    \exp\left(-\frac{2c_1D}{\sqrt{c_2}}\right)\nabla^2\phi^*(Y_0^i)\preceq \nabla^2\phi^*(Y_t^i)\preceq \exp\left(\frac{2c_1D}{\sqrt{c_2}}\right)\nabla^2\phi^*(Y_0^i)
\end{align*}
and 
\begin{align*}
    \exp\left(-\frac{2c_1D}{\sqrt{c_2}}\right)I_d\preceq \nabla^2\phi(X_0^i)^{1/2}[\nabla^2\phi(X_t^i)]^{-1}\nabla^2\phi(X_0^i)^{1/2}\preceq \exp\left(\frac{2c_1D}{\sqrt{c_2}}\right)I_d.
\end{align*}
Thus, we deduce that, in expectation with respect to $Z_0^i$, $M$ is upper bounded by \[
M = \frac{1-\exp(-\frac{1}{16c_1^2t})}{(1-c_1(tM_1 + 2\sqrt{\lambda td}))^2} + \exp\left(-\frac{1}{16c_1^2t} + \frac{2c_1D}{\sqrt{c_2}}\right),
\]
which goes to $1$ as $t\to 0$. This concludes the proof.
\end{proof}

\subsection{Proof of Theorem \ref{thm:main}}\label{app:thm_main}

\textit{Proof of Theorem \ref{thm:main}.}
We combine the analyses of \citet{jiang2021mirror} and \citet{nitanda2024improved}. We abuse notation to identify the probability distribution with its density function with respect to the Lebesgue measure. For instance, we denote by $\mu_*^{(N)}(\mathbf{x})$ the density of $\mu_*^{(N)}$.

Denote by $\mu_{0t}(\mathbf{x}_0,\mathbf{x}_t)$ the joint probability distribution of $(\mathbf{X}_0,\mathbf{X}_t)$ for time $t$, and by $\nu_{t|0}$, $\nu_{0|t}$, and $\nu_0$, $\nu_t$ the conditional and marginal distributions. We see that $\nu_0 = \mu_k^{(N)} = \operatorname{Law}(\mathbf{X}_k)$, $\nu_\eta = \mu_{k+1}^{(N)}=\operatorname{Law}(\mathbf{X}_{k+1})$, and \[
\nu_{0t}(\mathbf{x}_0,\mathbf{x}_t) = \nu_0(\mathbf{x}_0)\nu_{t|0}(\mathbf{x}_t|\mathbf{x}_0) = \nu_t(\mathbf{x}_t)\nu_{0|t}(\mathbf{x}_0|\mathbf{x}_t).
\]
Then, using $F(\mathbf{x}_0) = F(\mu_{\mathbf{x}_0})$ and the resulting equality
\begin{align*}
N\nabla_{x^i}F(\mathbf{x}_0) = \nabla \frac{\delta F(\mu_{\mathbf{x}_0})}{\delta\mu}(x_0^i),
\end{align*}
we also correspondingly obtain $G_0(\mathbf{x})$ and $\pi_0(\mathbf{x})$ as the diffusion and drift terms at time $t$ when $\mathbf{x}_t = \mathbf{x}$ with $\mathbf{x}_0$ at time $t = 0$.\footnote{Here, note that $G_0(\mathbf{x})\in\mathbb{R}^{dN\times dN}$ is the block-diagonal covariance matrix.} Then using the Fokker-Planck equation for the conditional density $\nu_{t|0}(\mathbf{x}_t|\mathbf{x}_0)$ \citep[Lemma 3]{wibisono2019proximal} and the argument in \citet{jiang2021mirror}, we obtain the Fokker-Planck equation of $\nu_t$: \begin{align}
    &\frac{\partial \nu_t(\mathbf{x})}{\partial t}\notag\\
    &= \int \frac{\partial\nu_{t|0}(\mathbf{x}|\mathbf{x}_0)}{\partial t}\nu_0(\mathbf{x}_0)d\mathbf{x}_0\notag \\
    &= \int[-\lambda\nabla \cdot (\nu_{t|0}(\nabla \cdot G_0(\mathbf{x}) - G_0(\mathbf{x})N\nabla F(\mathbf{x})) + \lambda\langle \nabla^2, \nu_{t|0}G_0(\mathbf{x})\rangle\notag \\
    &\qquad - \nabla\cdot (\nu_{t|0}\pi_0(\mathbf{x}))]\nu_0(\mathbf{x}_0)d\mathbf{x}_0\notag\\
    &= \lambda \nabla \cdot \left(\nu_{0|t}\int-(\nu_t(\nabla\cdot G_0(\mathbf{x})-G_0(\mathbf{x})N\nabla F(\mathbf{x}))) + \nabla \cdot(\nu_tG_0(\mathbf{x}))d\mathbf{x}_0\right)\notag\\
    &\qquad - \nabla \cdot \left(\nu_t\int \nu_{0|t}\pi_0(\mathbf{x})d\mathbf{x}_0\right)\notag\\
    &= \lambda \nabla \cdot \left(\nu_{0|t}\int G_0(\mathbf{x})\nabla \nu_t+ \nu_tG_0(\mathbf{x})N\nabla F(\mathbf{x})d\mathbf{x}_0)\right) - \nabla \cdot \left(\nu_t \int \nu_{0|t}\pi_0(\mathbf{x})d\mathbf{x}_0\right)\notag\\
    &= \lambda\nabla \cdot \left(\nu_{0|t}\int\left(\nu_t G_0(\mathbf{x})\nabla \log\frac{\nu_t}{\mu_*^{(N)}}(\mathbf{x})\right)d\mathbf{x}_0\right)- \nabla \cdot 
    \left(\nu_t\int\nu_{0|t}\pi_0(\mathbf{x})d\mathbf{x}_0\right).\label{eq:condition_pde}
\end{align}
Now we can control the decrease of the objective as
\begin{align}
    \frac{d\mathcal{L}^{(N)}}{dt}(\nu_t) &= \int \frac{\delta\mathcal{L}^{(N)}(\nu_t)}{\delta\mu^{(N)}}(\mathbf{x})\frac{\partial\nu_t}{\partial t}(\mathbf{x})d\mathbf{x}\notag\\
    &= \lambda \int \frac{\delta\mathcal{L}^{(N)}(\nu_t)}{\delta\mu^{(N)}}(\mathbf{x})\nabla \cdot \left(\nu_{0|t}\int\nu_t G_0(\mathbf{x})\nabla \log\frac{\nu_t}{\mu_*^{(N)}}(\mathbf{x})d\mathbf{x}_0\right)d\mathbf{x}\notag\\
    &\qquad - \int \frac{\delta\mathcal{L}^{(N)}(\nu_t)}{\delta\mu^{(N)}}(\mathbf{x})\nabla \cdot 
    \left(\nu_t\int\nu_{0|t}\pi_0(\mathbf{x})d\mathbf{x}_0\right)d\mathbf{x}\notag\\
    &= -\lambda \int \nu_{0|t}\int\left\langle\nabla\frac{\delta\mathcal{L}^{(N)}(\nu_t)}{\delta\mu^{(N)}},\nu_t G_0\nabla \log\frac{\nu_t}{\mu_*^{(N)}}\right\rangle d\mathbf{x}_0d\mathbf{x}\notag\\
    &\qquad + \int \nu_t\left\langle\nabla\frac{\delta\mathcal{L}^{(N)}(\nu_t)}{\delta\mu^{(N)}}, \int\nu_{0|t}\pi_0d\mathbf{x}_0\right\rangle d\mathbf{x}\notag\\
    &= -\lambda^2 \int \nu_{0|t}\int\left\langle\nabla \log\frac{\nu_t}{\mu_*^{(N)}},\nu_t G_0\nabla \log\frac{\nu_t}{\mu_*^{(N)}}\right\rangle d\mathbf{x}_0d\mathbf{x}\notag\\
    &\qquad + \lambda\int \nu_t\left\langle\nabla \log\frac{\nu_t}{\mu_*^{(N)}}, \int\nu_{0|t}\pi_0d\mathbf{x}_0\right\rangle d\mathbf{x}\notag\\
    &= -\lambda^2\underset{\nu_{t}}{\mathbb{E}}\bigg\|\nabla \log\frac{\nu_t}{\mu_*^{(N)}}\bigg\|_G^2+ \lambda\underset{\nu_{0t}}{\mathbb{E}}\left\langle \pi, \nabla \log\frac{\nu_t}{\mu_*^{(N)}}\right\rangle\label{eq:for_refined_analysis}\\
&\le -\lambda^2\underset{\nu_{t}}{\mathbb{E}}\bigg\|\nabla \log\frac{\nu_t}{\mu_*^{(N)}}\bigg\|_{[\nabla^2\phi]^{-1}}^2+ \frac{1}{2}\underset{\nu_{0t}}{\mathbb{E}}\|\pi\|_{\nabla^2\phi}^2 + \frac{\lambda^2}{2}\underset{\nu_t}{\mathbb{E}}\left\|\nabla \log\frac{\nu_t}{\mu_*^{(N)}}\right\|^2_{[\nabla^2\phi]^{-1}} \notag\\
     &\le -\lambda^2\CLSI{\KL(\nu_t\sep\mu_*^{(N)})} +  \frac{1}{2}\underset{\nu_{0t}}{\mathbb{E}}\|\pi\|^2_{\nabla^2\phi}\notag\\
     &\le -\lambda\CLSI\left(\mathcal{L}^{(N)}(\nu_t) - \mathcal{L}^{(N)}(\mu_*^{(N)})\right) + \frac{1}{2}\underset{\nu_{0t}}{\mathbb{E}}\|\pi\|^2_{\nabla^2\phi},\label{eq:dL_dt}
\end{align}
using \eqref{eq:condition_pde}, integration by parts, Young's inequality, and Lemma \ref{lem:entropy_sandwich}.

Next, using Assumptions  \ref{asmp:rel_lip_and_smooth}, \ref{asmp:self_concordance}, and \eqref{eq:pre_weighted}, for
\begin{align*}
M = \exp\left(\frac{2c_1D}{\sqrt{c_2}}\right),\quad \xi_t = \frac{1-\exp(-\frac{1}{16c_1^2t})}{(1-c_1(tM_1 + 2\sqrt{td}))^2} + \exp\left(-\frac{1}{16c_1^2t} + \frac{2c_1D}{\sqrt{c_2}}\right),
\end{align*}
we may bound
\begin{align}
    \underset{\nu_{0t}}{\mathbb{E}}\|\pi\|^2_{\nabla^2\phi}    &=  \underset{(\mathbf{x}_0,\mathbf{x}_t)\sim\nu_{0t}}{\mathbb{E}} \left[\sum_{i=1}^N\|\pi^i\|_{\nabla^2\phi(\mathbf{x}_t^i)}^2\right]\notag \\
    &\le M_2^2\underset{(\mathbf{x}_0,\mathbf{x}_t)\sim\nu_{0t}}{\mathbb{E}}\left[\sum_{i=1}^N\|\nabla\phi(\mathbf{x}_0^i)-\nabla\phi(\mathbf{x}_t^i)\|_{[\nabla^2\phi(\mathbf{x}_t^i)]^{-1}}^2\right]\notag\\
    &\le M_2^2\underset{(\mathbf{x}_0,\mathbf{x}_t)\sim\nu_{0t}}{\mathbb{E}}\left[\sum_{i=1}^N\left\|\nabla\frac{\delta F(\mu_{\mathbf{x}_0})}{\delta\mu}(\mathbf{x}_0^i) - \nabla\frac{\delta F(\mu_{\mathbf{x}_t})}{\delta\mu}(\mathbf{x}_t^i)\right\|_{[\nabla^2\phi(\mathbf{x}_t^i)]^{-1}}^2\right]\notag\\
    &\le M_2^4\underset{(\mathbf{x}_0,\mathbf{x}_t)\sim\nu_{0t}}{\mathbb{E}}\left[NW_{2,\phi}^2(\mu_{\mathbf{x}_0},\mu_{\mathbf{x}_t}) + \sum_{i=1}^N\left\|\mathbf{x}_0^i - \mathbf{x}_t^i\right\|_{[\nabla^2\phi(\mathbf{x}_t^i)]^{-1}}^2\right]\notag\\
    &\le 2M_2^4\underset{(\mathbf{x}_0,\mathbf{x}_t)\sim\nu_{0t}}{\mathbb{E}}\left[ \sum_{i=1}^N\left\|\mathbf{x}_0^i - \mathbf{x}_t^i\right\|_{[\nabla^2\phi(\mathbf{x}_t^i)]^{-1}}^2\right]\notag\\
    &\le 2M_2^4\underset{(\mathbf{x}_0,\mathbf{x}_t)\sim\nu_{0t}}{\mathbb{E}}\left[ \sum_{i=1}^N\left\|-t\nabla \frac{\delta F(\mu_{\mathbf{x}_0})}{\delta\mu}(\mathbf{x}_0^i) + \!\sqrt{2\lambda}\int_0^t[\nabla^2\phi(\mathbf{x}_s^i)]^{1/2}dB_s\right\|_{[\nabla^2\phi(\mathbf{x}_t^i)]^{-1}}^2\right]\notag\\
    &\le 4Nt^2M_1^2M_2^4\xi_t + 8\lambda tdNM_2^4M\notag\\
    &= 4NM_2^4(t^2M_1^2\xi_t + 2\lambda tdM).\label{eq:pi_bound}
\end{align}
Here we have used Assumption \ref{asmp:rel_lip_and_smooth}, It\^o's isometry, Jensen's inequality, and Lemma \ref{lem:jiang4}.

Now note that $\xi_t\le M$ deterministically, so for $t\in [0, \eta]$, we have
\begin{align*}
\frac{1}{2}\mathbb{E}_{\nu_{0t}}\|\pi\|_{\nabla^2\phi}^2\le N\delta_\eta, \quad\text{where}\quad \delta_\eta:=2\eta M_2^4M(\eta M_1^2 + 2\lambda d).
\end{align*}
By combining \eqref{eq:dL_dt} and \eqref{eq:pi_bound}, we see that \begin{align*}
    \frac{d}{dt}\left(\mathcal{L}^{(N)}(\nu_t) - \mathcal{L}^{(N)}(\mu_*^{(N)}) -\frac{N\delta_\eta}{2\CLSI\lambda}\right)\le -\CLSI\lambda \left(\mathcal{L}^{(N)}(\nu_t) - \mathcal{L}^{(N)}(\mu_*^{(N)}) -\frac{N\delta_\eta}{2\CLSI\lambda}\right).
\end{align*}
Noting that $\nu_\eta =\mu_{k+1}^{(N)}$ and $\nu_0 = \mu_{k}^{(N)}$, Gr\"onwall's inequality yields \[
\mathcal{L}^{(N)}(\nu_{k+1}) - \mathcal{L}^{(N)}(\mu_*^{(N)}) -\frac{N\delta_\eta}{2\CLSI\lambda} \le \exp\left(-\CLSI\lambda\eta\right)\left(\mathcal{L}^{(N)}(\nu_k) - \mathcal{L}^{(N)}(\mu_*^{(N)}) -\frac{N\delta_\eta}{2\CLSI\lambda}\right).
\]
Iterating yields \[
\mathcal{L}^{(N)}(\nu_{k}) - \mathcal{L}^{(N)}(\mu_*^{(N)}) -\frac{N\delta_\eta}{2\CLSI\lambda} \le \exp\left(-\CLSI\lambda\eta k\right)\left(\mathcal{L}^{(N)}(\nu_0) - \mathcal{L}^{(N)}(\mu_*^{(N)}) -\frac{N\delta_\eta}{2\CLSI\lambda}\right),
\]
from which the claim follows.
\qed

\subsection{Proof of Theorem \ref{thm:stoch_grad}}\label{app:thm_stoch_grad}
\begin{proof}
    We follow the proof of Theorem \ref{thm:main}, using a similar notation, and start with \eqref{eq:dL_dt}:\begin{align}
        \frac{d\mathcal{L}^{(N)}}{dt}(\nu_t) \le -\lambda\CLSI\left(\mathcal{L}^{(N)}(\nu_t) - \mathcal{L}^{(N)}(\mu_*^{(N)})\right) + \frac{1}{2}\underset{\nu_{0t}}{\mathbb{E}}\|\widetilde{\pi}\|^2_{\nabla^2\phi},\label{eq:dL_dt_stoch}
    \end{align}
    where 
    $\widetilde{\pi}^i = G_t^i(X_t^i)\left(\nabla\frac{\delta F(\mu_t)}{\delta\mu}(X_t^i)-g_0^i\right)$
    is the stochastic gradient variation in Lemma \ref{lem:jiang4}. Continuing the proof, we have 
    \begin{align}
    \underset{\nu_{0t}}{\mathbb{E}}\|\widetilde{\pi}\|^2_{\nabla^2\phi} &=\underset{(\mathbf{x}_0,\mathbf{x}_t)\sim\nu_{0t}}{\mathbb{E}} \left[\sum_{i=1}^N\|\widetilde{\pi}^i\|_{\nabla^2\phi(\mathbf{x}_t^i)}^2\right]\notag \\
    &= \underset{(\mathbf{x}_0,\mathbf{x}_t)\sim\nu_{0t}}{\mathbb{E}} \left[\sum_{i=1}^N\left\|G_t^i(\mathbf{x}_t^i)\left(\nabla\frac{\delta F(\mu_{\mathbf{x}_t})}{\delta\mu}(\mathbf{x}_t^i)-g_0^i\right)\right\|_{\nabla^2\phi(\mathbf{x}_t^i)}^2\right]\notag\\
    &\le 2\E[\|\pi\|_{\nabla^2\phi}^2] + 2\underset{(\mathbf{x}_0,\mathbf{x}_t)\sim\nu_{0t}}{\mathbb{E}} \left[\sum_{i=1}^N\left\|G_t^i(\mathbf{x}_t^i)\left(\nabla\frac{\delta F(\mu_{\mathbf{x}_t})}{\delta\mu}(\mathbf{x}_0^i)-g_0^i\right)\right\|_{\nabla^2\phi(\mathbf{x}_t^i)}^2\right],\notag
\end{align}
where we use $\|a+b\|^2\le 2\|a\|^2+2\|b\|^2$. The first term is deterministically bounded by $4N\delta_\eta$, with $\delta_\eta:=2\eta M_2^4M(\eta M_1^2 + 2\lambda d)$ (Theorem \ref{thm:main}). For the second term, we have \begin{align*}
    &\underset{(\mathbf{x}_0,\mathbf{x}_t)\sim\nu_{0t}}{\mathbb{E}} \left[\sum_{i=1}^N\left\|G_t^i(\mathbf{x}_t^i)\left(\nabla\frac{\delta F(\mu_{\mathbf{x}_t})}{\delta\mu}(\mathbf{x}_0^i)-g_0^i\right)\right\|_{\nabla^2\phi(\mathbf{x}_t^i)}^2\right]\\
    &\qquad= \underset{(\mathbf{x}_0,\mathbf{x}_t)\sim\nu_{0t}}{\mathbb{E}} \left[\sum_{i=1}^N\left\|\nabla\frac{\delta F(\mu_{\mathbf{x}_t})}{\delta\mu}(\mathbf{x}_0^i)-g_0^i\right\|_{[\nabla^2\phi(\mathbf{x}_t^i)]^{-1}}^2\right]\\
    &\qquad\le \frac{N\sigma^2}{c_2}.
\end{align*}
Then combining the above displays yields \[
\frac{1}{2}\underset{\nu_{0t}}{\mathbb{E}}\|\widetilde{\pi}\|^2_{\nabla^2\phi} \le N\left(2\delta_\eta + \frac{\sigma^2}{c_2}\right).
\]
By completing the argument as in Theorem \ref{thm:main}, we deduce 
\[
\mathcal{L}^{(N)}(\nu_{k}) - \mathcal{L}^{(N)}(\mu_*^{(N)}) -\frac{N(2\delta_\eta + \frac{\sigma^2}{c_2})}{2\CLSI\lambda} \le \exp\left(-\CLSI\lambda\eta k\right)\left(\mathcal{L}^{(N)}(\nu_0) - \mathcal{L}^{(N)}(\mu_*^{(N)}) -\frac{N(2\delta_\eta + \frac{\sigma^2}{c_2})}{2\CLSI\lambda}\right),
\]
which concludes the proof.
\end{proof}

\section{Additional Experiments}\label{app:additional_experiments}

In this section, we scale up the experiments in Section~\ref{sec:exp} to high-dimensional settings. We observe similar gains as before, verifying the scalability and effectiveness of mirror MFLD.
\begin{figure}[!h]
    \centering
    \includegraphics[width=0.4\linewidth]{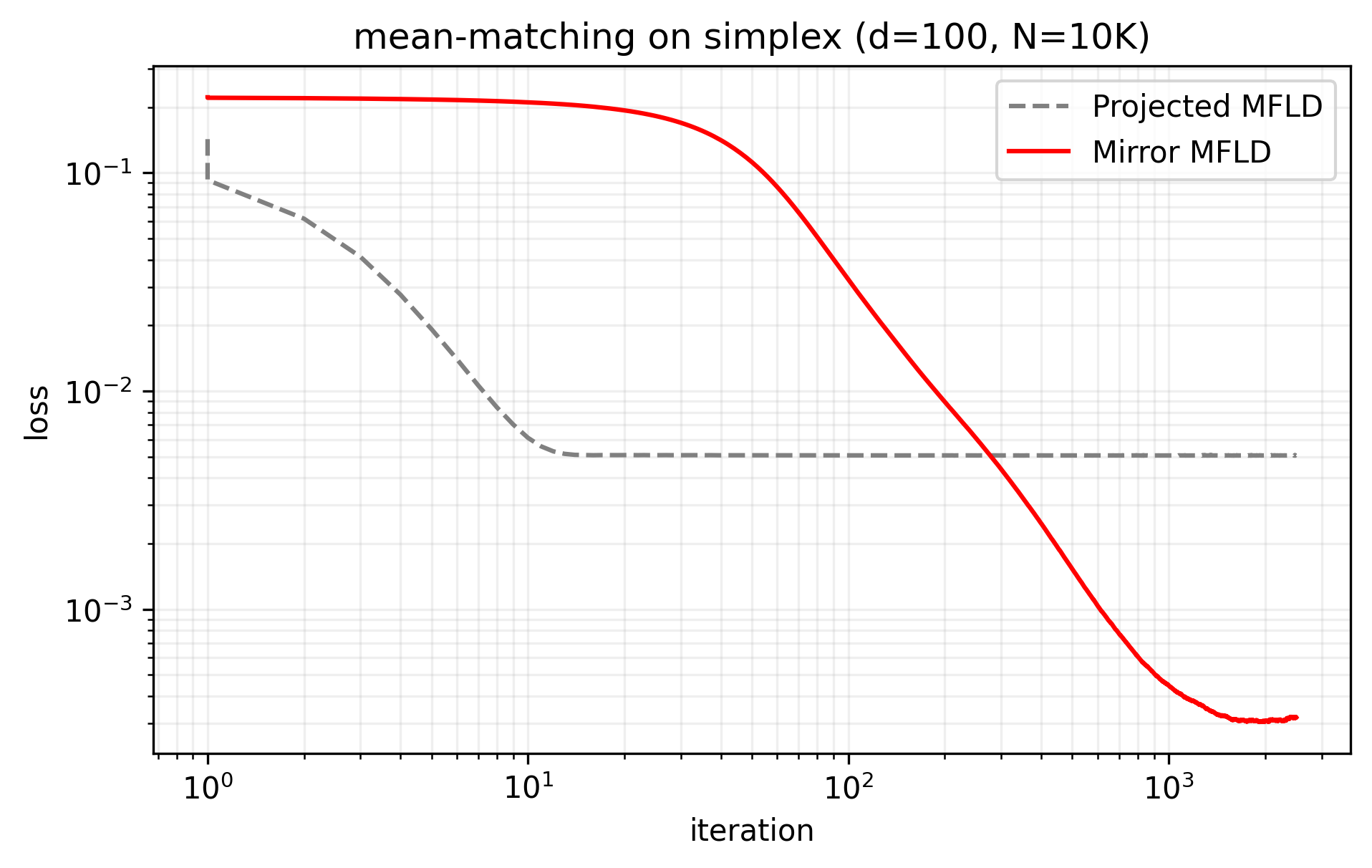}
    \caption{Optimization over the unit simplex, with dimension $d=100$ and randomly sampled Dirichlet mean-matching target. Both mirror and projected MFLD are run using 10K particles with $\eta=0.1$, $\beta=10^{-9}$ over 2500 iterations (log-log plot). Similarly to the low-dimensional experiment, projected MFLD initially converges faster but our mirror MFLD algorithm is able to achieve significantly smaller loss, converging to a more optimal solution. Regarding scalability, the experiment runs in under 10 minutes on a single GPU.}
\end{figure}

\begin{figure}[!h]
    \centering
    \includegraphics[width=0.28\linewidth]{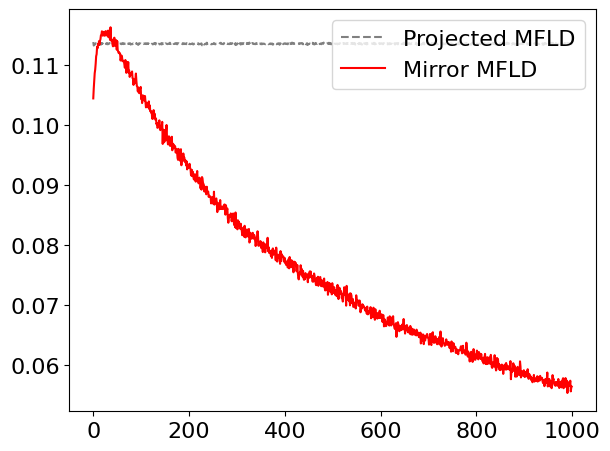}
    \caption{Optimization over the spectraplex $\mathbb{S}_+^d$, with dimension $d=50$. Both algorithms are run with $\eta = 0.3$, $\lambda = 0.1$, $\gamma = 0.02$. Again, we observe projected MFLD essentially is unable to optimize due to the geometry of the set. We consider slightly lower dimensionality as the dimension complexity of this problem scales as $\Theta(d^3)$; this experiment runs in 90 minutes on a CPU.}
\end{figure}

\begin{figure}[!h]
    \centering
    \includegraphics[width=0.4\linewidth]{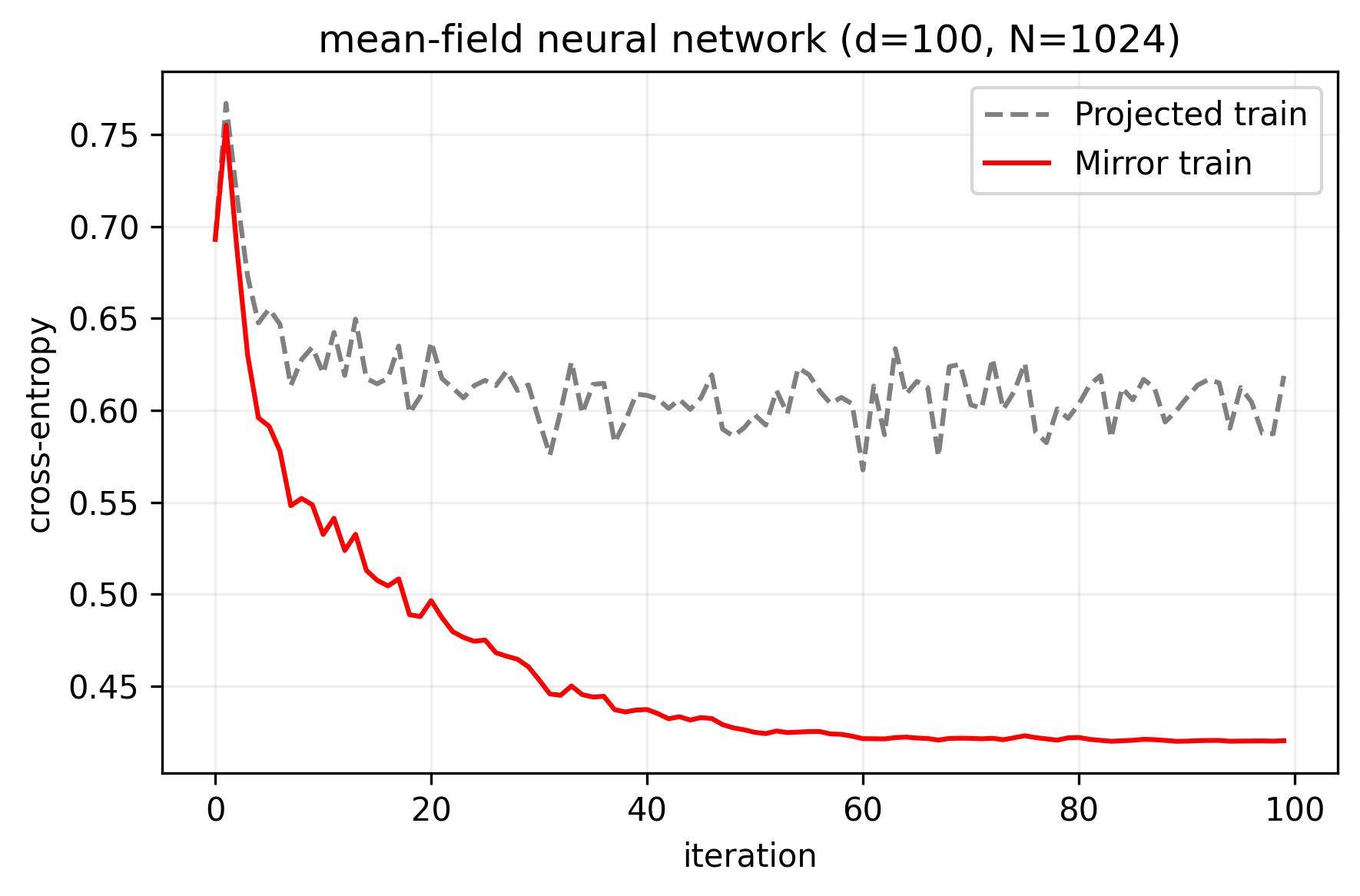}
    \caption{Optimizing a norm-constrained mean-field neural network for classification, with $d=100$, ReLU activations and hidden layer width $1024$. Both algorithms are run with $\eta=0.1$, $\lambda=10^{-3}$ and weights constrained to the unit ball. As in the low-dimensional case, MMFLD decreases the loss at a significantly faster rate than projected MFLD, which stagnates earlier. The experiment runs in under $2$ minutes on a single GPU.}
\end{figure}


\end{document}